\documentclass[11pt]{article}

\usepackage{epsfig,epsf,fancybox}
\usepackage{amsmath}
\usepackage{mathrsfs}
\usepackage{amssymb}
\usepackage{graphicx}
\usepackage{color}
\usepackage{multirow}
\usepackage{paralist}
\usepackage{verbatim}
\usepackage{galois}
\usepackage{algorithm}
\usepackage{algorithmic}
\usepackage{boxedminipage}
\usepackage{booktabs}
\usepackage{stmaryrd}
\usepackage{subfig}
\usepackage{wrapfig}
\usepackage[bottom]{footmisc}

\usepackage{hyperref}

\usepackage{nicefrac}
\usepackage[left=1in,top=1in,right=1in,bottom=1in,letterpaper]{geometry}

\newtheorem{theorem}{Theorem}[section]

\newtheorem{lemma}[theorem]{Lemma}
\newtheorem{proposition}[theorem]{Proposition}
\newtheorem{definition}[theorem]{Definition}
\newtheorem{remark}[theorem]{Remark}

\numberwithin{equation}{section}

\providecommand{\keywords}[1]{\textbf{\textit{Keywords---}} #1}

\newcommand{\T}{\mathrm{T}}
\newcommand{\st}{\textnormal{s.t.}}

\newcommand{\Tr}{\textnormal{Tr}}

\newcommand{\diag}{\mathrm{diag}}
\newcommand{\grad}{\textnormal{grad}\,}

\newcommand{\subdiff}{\textnormal{subdiff}\,}
\newcommand{\Diag}{\mathrm{Diag}\,}

\newcommand{\retr}{\textnormal{Retr}}

\newcommand{\sign}{\textnormal{sign}}

\newcommand{\argmin}{\mathop{\rm argmin}}

\newcommand{\KCal}{\mathcal{K}}

\newcommand{\SCal}{\mathcal{S}}

\newcommand{\proj}{\textnormal{Proj}}
\newcommand{\Proj}{\textnormal{Proj}}

\newcommand{\dist}{\textnormal{dist}}
\newcommand{\St}{\textnormal{St}}
\newcommand{\Tg}{\textnormal{T}}

\newcommand{\br}{\mathbb{R}}

\newcommand{\be}{\begin{equation}}
\newcommand{\ee}{\end{equation}}
\newcommand{\ba}{\begin{array}}
\newcommand{\ea}{\end{array}}
\newcommand{\bad}{\begin{aligned}}
\newcommand{\ead}{\end{aligned}}

\newcommand{\WCal}{\mathcal{W}}

\newcommand{\MCal}{\mathcal{M}}
\newcommand{\M}{\mathcal{M}}

\newcommand{\PCal}{\mathcal{P}}

\newcommand{\NCal}{\mathcal{N}}

\newcommand{\UCal}{\mathcal{U}}

\newcommand{\bigO}{O}

\newcommand{\mydefn}{:=}
\newcommand{\etal}{{\it et al.\ }}

\newcommand{\regot}{\textsc{RegOT}}
\newcommand{\Retr}{\mathrm{Retr}}
\newcommand{\onebf}{\mathbf{1}}

\begin{document}

\title{A Riemannian Block Coordinate Descent Method for Computing the Projection Robust Wasserstein Distance}

\author{Minhui Huang\thanks{Department of Electrical and Computer Engineering, University of California, Davis}
\and Shiqian Ma\thanks{Department of Mathematics, University of California, Davis}
\and Lifeng Lai\footnotemark[1]}
\date{\today}
\maketitle

\begin{abstract}
The Wasserstein distance has become increasingly important in machine learning and deep learning. Despite its popularity, the Wasserstein distance is hard to approximate because of the curse of dimensionality. A recently proposed approach to alleviate the curse of dimensionality is to project the sampled data from the high dimensional probability distribution onto a lower-dimensional subspace, and then compute the Wasserstein distance between the projected data. However, this approach requires to solve a max-min problem over the Stiefel manifold, which is very challenging in practice. The only existing work that solves this problem directly is the RGAS (Riemannian Gradient Ascent with Sinkhorn Iteration) algorithm, which requires to solve an entropy-regularized optimal transport problem in each iteration, and thus can be costly for large-scale problems. In this paper, we propose a Riemannian block coordinate descent (RBCD) method to solve this problem, which is based on a novel reformulation of the regularized max-min problem over the Stiefel manifold. We show that the complexity of arithmetic operations for RBCD to obtain an $\epsilon$-stationary point is $O(\epsilon^{-3})$. This significantly improves the corresponding complexity of RGAS, which is $O(\epsilon^{-12})$. Moreover, our RBCD has very low per-iteration complexity, and hence is suitable for large-scale problems. Numerical results on both synthetic and real datasets demonstrate that our method is more efficient than existing methods, especially when the number of sampled data is very large.

\end{abstract}


\section{Introduction}

The Wasserstein distance measures the closeness of two probability distributions on a given metric space. It has wide applications in machine learning problems, including the latent mixture models \cite{ho2016convergence}, representation learning \cite{ozair2019wasserstein}, reinforcement learning \cite{bellemare2017distributional} and stochastic optimization \cite{nagaraj2019sgd}. Intuitively, the Wasserstein distance is the minimum cost of turning one distribution into the other. To calculate the Wasserstein distance, one is required to solve an optimal transport (OT) problem, which has been widely adopted in machine learning and data science.

However, it is known that the sample complexity of approximating Wasserstein distances using only samples can grow exponentially in dimension \cite{dudley1969speed, fournier2015rate, weed2019sharp, Lei-2020}. This leads to very large-scale OT problems that are challenging to solve using traditional approaches. As a result, this has motivated research on mitigating this curse of dimensionality when approximating Wasserstein distance using OT.
One approach for reducing the dimensionality is the sliced approximation of OT proposed by Rabin \etal \cite{rabin2011wasserstein}. This approach projects the clouds of points from two probability distributions onto a given line, and then computes the OT cost between these projected values as an approximation to the original OT cost. This idea has been further studied in \cite{kolouri2016sliced,bonneel2015sliced,kolouri2019generalized, deshpande2019max} for defining kernels, computing barycenters, and training generative models.
Recently, motivated by the sliced approximation of OT, Paty and Cuturi \cite{paty2019subspace} and Niles-Weed and Rigollet \cite{niles2019estimation} proposed to project the distance measures onto $k$-dimensional subspaces. The $k$-dimensional subpsaces are obtained by maximizing the Wasserstein distance between two measures after projection. The approach is called Wasserstein projection pursuit (WPP), and the largest Wasserstein distance between the two measures after projection onto the $k$-dimensional subspaces is called the projection robust Wasserstein distance (PRW). As proved in \cite{niles2019estimation} and \cite{lin2020projection2}, WPP/PRW indeed reduces the sample complexity and resolves the issue of curse of dimensionality for the spiked transport model. However, computing PRW requires to solve a nonconvex max-min problem over the Stiefel manifold, which demands {efficient} algorithms. In this paper, we propose a novel algorithm that can compute PRW efficiently and faithfully.

In the case of discrete probability measures, one is given two sets of finite number atoms, $\{x_1, x_2, \ldots, x_n\} \subset \br^d$ and  $\{y_1, y_2, \ldots, y_n\} \subset \br^d$, and two probability distributions $\mu_n = \sum_{i=1}^n r_i\delta_{x_i}$ and $\nu_n = \sum_{j=1}^n c_j\delta_{y_j}$. Here $r = (r_1, r_2, \ldots, r_n)^\top \in \Delta^n$ and $c = (c_1, c_2, \ldots, c_n)^\top \in \Delta^n$,
$\Delta^n$ denotes the probability simplex in $\br^n$ and $\delta_x$ denotes the Dirac delta function at $x$. Computing the Wasserstern distance between $\mu_n$ and $\nu_n$ is equivalent to solving an OT problem \cite{villani2008optimal}:
\be \label{OT}
\WCal^2(\mu_n, \nu_n) \ = \ \min_{\pi \in \Pi(\mu_n, \nu_n)} \langle C, \pi \rangle,
\ee
where the transporation polytope $\Pi(\mu_n,\nu_n):=\{\pi\in\br^{n\times n}_+\mid \pi\onebf = r, \pi^\top\onebf = c\}$, and $\onebf$ denotes the $n$-dimensional all-one vector. Throughout this paper, $C \in\br^{n\times n}$ denotes the matrix whose $(i,j)$-th component is $C_{ij} =  \|x_i - y_j\|^2$.
Computing the PRW can then be formulated as the following max-min problem \cite{paty2019subspace}:
\be\label{PRW}
\PCal_k^2(\mu_n,\nu_n):=\max_{U\in\M}\min_{\pi\in\Pi(\mu_n,\nu_n)} f(\pi,U) :=\sum_{i,j=1}^n\pi_{ij}\|U^\top x_i-U^\top y_j\|^2.
\ee
Throughout this paper, $\M$ denotes the Stiefel manifold $\M\equiv\St(d,k):=\{U\in\br^{d\times k}\mid U^\top U = I_{k\times k}\}$, which is a sub-manifold embeded in the ambient Euclidean space $\br^{d\times k}$.
Here integer $k\in [d]$, where $[d]$ denotes the set of integers $\{1,2,\ldots,d\}$. Therefore, $\|U^\top x_i-U^\top y_j\|^2$ denotes the distance measure that is projected on the $k$-dimensional subspace with the columns of $U$ being a basis.
However, due to its nonconvex nature, solving \eqref{PRW} is not an easy task. In fact, Paty and Cuturi \cite{paty2019subspace} concluded that the PRW \eqref{PRW} is difficult to compute, and they proposed to study its corresponding dual problem -- the subspace robust Wasserstein distance (SRW):
\be\label{SRW}
\SCal_k^2(\mu_n,\nu_n):=\min_{\pi\in\Pi(\mu_n,\nu_n)}\max_{U\in\M} f(\pi,U) :=\sum_{i,j=1}^n\pi_{ij}\|U^\top x_i-U^\top y_j\|^2.
\ee
It is shown in \cite{paty2019subspace} that the SRW \eqref{SRW} is equivalent to:
\be\label{SRW-equiv}
\SCal_k^2(\mu_n,\nu_n) = \max_{0\preceq \Omega\preceq I, \Tr(\Omega)=k}s(\Omega) := \min_{\pi\in\Pi(\mu_n,\nu_n)}\sum_{ij}\pi_{ij}(x_i-y_j)^\top\Omega(x_i-y_j),
\ee
which can be viewed as maximizing the concave function $s(\Omega)$ over the convex set $\{\Omega\mid 0\preceq \Omega\preceq I, \Tr(\Omega)=k\}$, where $\Tr(\Omega)$ denotes the trace of matrix $\Omega$. Problem \eqref{SRW-equiv} is a convex optimization problem and thus numerically more tractable. Paty and Cuturi \cite{paty2019subspace} proposed a projected subgradient method for solving \eqref{SRW}, and in each iteration computing the subgradient of $s$ requires solving an OT problem in the form of \eqref{OT}. To improve the computational efficiency, they also proposed a Frank-Wolfe method for solving the following entropy-regularized SRW:
\be\label{SRW-equiv-reg}
\max_{0\preceq \Omega\preceq I, \Tr(\Omega)=k}s_\eta(\Omega) := \min_{\pi\in\Pi(\mu_n,\nu_n)}\sum_{ij}\pi_{ij}(x_i-y_j)^\top\Omega(x_i-y_j) - \eta H(\pi),
\ee
where $H(\pi) = -\sum_{ij}(\pi_{ij}\log\pi_{ij}-\pi_{ij})$ is a {constant-shifted} entropy regularizer and $\eta>0$ is a weighting parameter.
Each iteration of the Frank-Wolfe method requires solving a regularized OT ($\regot$) problem in the following form:
\be \label{OT-reg}
\min_{\pi \in \Pi(\mu_n, \nu_n)} \langle M, \pi \rangle - \eta H(\pi),
\ee
for a given matrix $M \in \br^{n\times n}$. Solving \eqref{OT-reg} can be done more efficiently using the Sinkhorn's algorithm \cite{cuturi2013sinkhorn}.
However, note that solving \eqref{SRW} does not yield a solution to \eqref{PRW} because there exists a duality gap.

In a more recent work, Lin \etal \cite{lin2020projection} proposed a Riemannian gradient method to compute the PRW \eqref{PRW}. More specifically, they proposed the RGAS algorithm for computing the PRW with entropy regularization:
\be\label{PRW-reg}
\max_{U\in\M} p(U) := \min_{\pi\in\Pi(\mu_n,\nu_n)} f_\eta(\pi,U) :=\sum_{ij}\pi_{ij}\|U^\top x_i-U^\top y_j\|^2 - \eta H(\pi).
\ee
Lin \etal \cite{lin2020projection} proved that the RGAS algorithm combined with a rounding procedure (will be discussed later) gives an $\epsilon$-stationary point to the PRW problem \eqref{PRW}.

The details of RGAS are given in Algorithm \ref{alg:RGAS}. In Algorithm \ref{alg:RGAS}, $\grad p$ denotes the Riemannian gradient of function $p$, $\retr$  denotes a retraction operator on the manifold $\M$, and $\pi^{t+1}$ is the optimal solution to the $\regot$ problem \eqref{OT-reg} with $M_{ij} = \|(U^t)^\top(x_i-y_j)\|_2^2$. Note that computing $\xi^{t+1}$ in fact requires $\pi^{t+1}$, and the latter further requires to solve a $\regot$ problem \eqref{OT-reg}. This can be costly because an iterative solver for $\regot$ is needed in every iteration.

\begin{algorithm}[htbp]
\caption{RGAS for Computing PRW \cite{lin2020projection}}\label{alg:RGAS}
\begin{algorithmic}[1]
\STATE{Input: $\{(x_i,r_i)\}_{i\in [n]}$ and $\{(y_j,c_j)\}_{j\in [n]}$, $U^0\in\M$, and parameter $\eta$. }
\FOR{$t = 0, 1, \ldots $}
\STATE{Compute $\pi^{t+1}\leftarrow \regot(\{(x_i,r_i)\}_{i\in [n]},\{(y_j,c_j)\}_{j\in [n]}, U^t, \eta) $}
\STATE{Compute $\xi^{t+1} = \grad p(U^t)$}
\STATE{Compute $U^{t+1}\leftarrow\Retr_{U^t}(\tau \xi^{t+1})$}
\ENDFOR
\STATE{Output: $U^{t+1}$ and $\pi^{t+1}$}
\end{algorithmic}
\end{algorithm}

\paragraph{Our contributions.} In this paper, motivated by the demand for efficient algorithms for computing PRW \eqref{PRW}, we design a novel Riemannian block coordinate descent (RBCD) algorithm for solving this problem, and analyze its convergence behavior. Our main contributions of this paper lie in several folds.
\begin{enumerate}
\item We propose an equivalent formulation of \eqref{PRW-reg}, which consists a minimization problem only and thus is much easier to solve than the max-min problem \eqref{PRW-reg}.

\item We propose a RBCD algorithm for solving the equivalent formulation of \eqref{PRW-reg}. The per-iteration complexity of RBCD is much lower than the existing methods in \cite{paty2019subspace} and \cite{lin2020projection}, as it does not need to solve OT or $\regot$ problems. This  makes our algorithm suitable for large-scale problems.

\item We propose a variant of RBCD (named RABCD) that adopts an adaptive step size for the Riemannian gradient step. This stategy helps speed up the convergence of RBCD in practice.

\item We prove that the complexity of arithmetic operations of RBCD and RABCD are both $O(\epsilon^{-3})$ for obtaining an $\epsilon$-stationary point of problem \eqref{PRW}. This significantly improves the corresponding complexity of RGAS, which is $O(\epsilon^{-12})$.
\end{enumerate}

{\bf Organization.} The rest of this paper is organized as follows. In Section \ref{sec:prel}, we briefly review some necessary backgrounds of Riemannian optimization. In Section \ref{sec:alg}, we introduce our RBCD algorithm for computing the PRW. The complexity of arithmetic operations of RBCD for obtaining an $\epsilon$-stationary point of \eqref{PRW} is analyzed in Section \ref{sec:conv}. Section \ref{sec:RABCD} is dedicated to a variant of RBCD with adaptive step size, and its complexity analysis. We present numerical results on both synthetic and real datasets in Section \ref{sec:num} to demonstrate the advantages of our algorithms comparing with existing methods. Finally, we draw some conclusions in Section \ref{sec:con}.

\section{Preliminaries on Riemannian Optimization}\label{sec:prel}
In this section, we review a few important concepts in Riemannian optimization.
\begin{definition}[\cite{absil2009optimization}]
\textbf{(Tangent Space)} The tangent space of $\MCal$ at $U\in\M$ is defined as
\[
\Tg_U\MCal = \left\{ \gamma'(0) : \gamma \text{ is a smooth curve with } \gamma(0) = U, \gamma([-\iota, \iota]) \subset \MCal, \iota > 0\right\}.
\]
The tangent bundle is defined as $\Tg \MCal = \left\{ (U, \xi) : U \in \MCal, \xi \in \Tg_U\MCal\right\}$.
\end{definition}

For the Stiefel manifold $\M$, its tangent space at $U\in\M$ can be written as:
\[
\Tg_U\M \ \mydefn \ \left\{ \xi\in \br^{d\times k} \mid \xi^\top U + U^\top \xi = 0 \right\}.
\]
Throughout this paper, we consider the Riemannian metric on $\M$ that is induced from the Euclidean inner product; i.e., for any $\xi, \eta \in \T_U\M$, we have $\langle\xi,\eta\rangle_U=\Tr(\xi^\top\eta)$. With this choice of Riemannian metric, it is known that for any smooth function $f$, we have
\[\grad f(U) = \Proj_{\T_U\M}\nabla f(U).\]
That is, the Riemannian gradient of $f$ is equal to the orthogonal projection of the Euclidean gradient onto the tangent space.

\begin{definition}[\cite{absil2009optimization}]
\textbf{(Retraction)} A retraction on $\MCal$ is a smooth mapping $\retr(\cdot)$ from the tangent bundle $\Tg\MCal$ onto $\MCal$ satisfying the following two conditions:
\begin{itemize}
\item $\retr_U(0) = U$, $\forall U\in\M$, where $0$ denotes the zero element of $\Tg_U\MCal$;
\item For any $U \in \MCal$, it holds that
\[
\lim_{ \Tg_U\MCal \ni \xi \rightarrow 0} \|\retr_U(\xi) - (U + \xi)\|_F/\|\xi\|_F = 0.
\]
\end{itemize}
\end{definition}

For the Stiefel manifold, commonly used retraction operators inlcude the polar decomposition, the QR decomposition, and the Cayley transformation. We refer to \cite{chen2020proximal} for more details on these retraction operations. The retraction on the Stiefel manifold $\M$ has the following useful properties
\begin{proposition}[\cite{boumal2019global}]\label{pro:retr}
There exists constants $L_1, L_2 > 0$ such that for any $U \in \M$ and $\xi \in \Tg_U\M$, the following inequalities hold:
\begin{eqnarray*}
\|\retr_U(\xi) - U\|_F & \leq & L_1\|\xi\|_F, \\
\|\retr_U(\xi) - (U + \xi)\|_F & \leq & L_2\|\xi\|_F^2.
\end{eqnarray*}
\end{proposition}

\begin{remark}
The values of the constants $L_1, L_2$ depend on the manifold structure and may scale with dimensions for general manifolds. However, for retractions on the Stiefel manifold, these constants are independent of $(d, k)$ and can be computed explicitly \cite{jiang2017vector}[Proposition 3.1]. Specifically, when using the QR factorization as the retraction, we have $L_1 = 1 + \sqrt{2}/2, L_2 = \sqrt{10}$. When using the the polar decomposition as the retraction, $L_1 = 1 + \sqrt{2}/2, L_2 = \sqrt{10}.$
\end{remark}

\section{A Riemannian Block Coordinate Descent Algorithm for Computing the PRW}\label{sec:alg}

In this section, we present our RBCD algorithm for computing the PRW \eqref{PRW}. Our algorithm is based on a new reformulation of the entropy-regularized problem \eqref{PRW-reg}.
First, we introduce some notation for the ease of the presentation. We denote $\varphi(\pi) := \pi\onebf$, and $\kappa(\pi) :=\pi^\top\onebf$. The inner minimization problem in \eqref{PRW-reg} can be equivalently written as
\be\label{PRW-reg-inner-min}
\min_\pi \ \sum_{ij}\pi_{ij}\|U^\top x_i - U^\top y_j\|^2 - \eta H(\pi), \ \st, \ \varphi(\pi) = r, \kappa(\pi) = c,
\ee
which is a convex problem with respect to $\pi$. The Lagrangian dual problem of \eqref{PRW-reg-inner-min} is given by:
\be\label{PRW-reg-inner-min-dual}
\max_{\alpha,\beta} \min_{\pi}\sum_{ij} \pi_{ij}\|U^\top x_i - U^\top y_j\|^2 - \eta H(\pi) + \alpha^\top(\varphi(\pi) - r) + \beta^\top( \kappa(\pi) - c),
\ee
where $\alpha$ and $\beta$ denote the Lagrange multipliers of the two equality constraints. {For the minimization over $\pi$, we add a redundant constraint $\sum_{i,j} {\pi_{i,j}} = 1$ and consider the following problem:
\be\label{PRW-reg-inner-min-dual-min-redundant}
\min_{\sum_{i,j} {\pi_{i,j}} = 1}\sum_{ij} \pi_{ij}\|U^\top x_i - U^\top y_j\|^2 - \eta H(\pi) + \alpha^\top(\varphi(\pi) - r) + \beta^\top( \kappa(\pi) - c).
\ee
It is easy to verify that the optimal solution of \eqref{PRW-reg-inner-min-dual-min-redundant} is given by
{
\be\label{sol-pi}\pi_{ij} = \frac{\exp((-\alpha_i - \beta_j - \|U^\top x_i - U^\top y_j\|^2)/\eta)}{\sum_{i,j} \exp((-\alpha_i - \beta_j - \|U^\top x_i - U^\top y_j\|^2)/\eta)}.
\ee
}
Substituting \eqref{sol-pi} into \eqref{PRW-reg-inner-min-dual}, we know that \eqref{PRW-reg-inner-min-dual} is equivalent to
{
\be\label{PRW-reg-inner-min-dual-equiv}\max_{\alpha,\beta} \ - \eta \log\left( \sum_{ij}\exp\left(-\frac{\alpha_i+\beta_j+\|U^\top x_i - U^\top y_j\|^2}{\eta}\right) \right)-\sum_i r_i\alpha_i - \sum_j c_j\beta_j.
\ee
The purpose of adding the redundant constraint in \eqref{PRW-reg-inner-min-dual-min-redundant} is to guarantee that the objective function of \eqref{PRW-reg-inner-min-dual-equiv} is Lipschitz smooth.
}
By combining \eqref{PRW-reg-inner-min-dual-equiv} and \eqref{PRW-reg}, we know that the max-min problem \eqref{PRW-reg} is equivalent to the following maximization problem:
{
\be\label{PRW-reg-inner-min-dual-equiv-next}
\max_{U \in \M,\alpha, \beta} \ - \eta\log\left( \sum_{ij}\exp\left(-\frac{\alpha_i+\beta_j+\|U^\top x_i - U^\top y_j\|^2}{\eta}\right) \right)-\sum_i r_i\alpha_i - \sum_j c_j\beta_j.
\ee}

We now define $u = - \alpha / \eta, v = - \beta / \eta$, {function $\zeta(u,v,U) \in  \br^{n\times n}$
\be\label{def-zeta}
[\zeta(u,v,U)]_{ij} = \exp{\left(-\frac{1}{\eta}\|U^\top(x_i - y_j)\|^2 + u_i + v_j\right)},
\ee
and function $\pi(u,v,U) \in \br^{n\times n}$ with
\be\label{def-pi}
[\pi(u,v,U)]_{ij} := \frac{[\zeta(u,v,U)]_{ij}}{\|\zeta(u,v,U)\|_1}.
\ee}
Then \eqref{PRW-reg-inner-min-dual-equiv-next} is equivalent to the following Riemannian minimization problem:
{
\be\label{eq:dualprw}
\min_{U \in \M, u, v \in \br^n} g(u, v, U) := \log\left( \|\zeta(u,v,U)\|_1 \right) - r^\top u - c^\top v.
\ee}
There are three block variables $(u,v,U)$ in \eqref{eq:dualprw}, and the objective function $g$ is a smooth function with respect to $(u,v,U)$. Moreover, for fixed $v$ and $U$, minimizing $g$ with respect to $u$ can be done analytically, and simliarly, for fixed $u$ and $U$, minimizing $g$ with respect to $v$ can also be done analytically. 
For fixed $u$ and $v$, minimizing $g$ with respect to $U$ is a Riemannian optimization problem with smooth objective function. Therefore, we propose a Riemannian block coordinate descent method for solving \eqref{eq:dualprw}, whose $t$-th iteration updates the iterates as follows:
{
\begin{subequations}\label{RBCD-3-steps}
\begin{align} 
u^{t+1} & \in \argmin_u g(u, v^{t}, U^{t}) \label{RBCD-3-steps-1} \\
v^{t+1} & \in \argmin_v g(u^{t+1}, v, U^{t})  \label{RBCD-3-steps-2} \\
V_{\pi(u^{t+1},v^{t+1},U^t)} & := \sum_{ij} [\pi(u^{t+1}, v^{t+1}, U^t)]_{ij} (x_i - y_j) (x_i - y_j)^\top \label{RBCD-3-steps-3-1} \\
\xi^{t+1} & := \grad_U g(u^{t+1},v^{t+1},U^t) = \proj_{\T_{U^t}\M} \left(-\frac{2}{\eta}V_{\pi(u^{t+1},v^{t+1},U^t)}U^t\right) \label{RBCD-3-steps-3-2} \\
U^{t+1} & := \retr_{U^t}( - \tau\xi^{t+1}), \label{RBCD-3-steps-3-3}
\end{align}
\end{subequations}}
where the notation $V_\pi$ is defined as: $V_\pi = \sum_{ij}\pi_{ij}(x_i-y_j)(x_i-y_j)^\top\in\br^{d\times d}$. Note that the minimization problems in \eqref{RBCD-3-steps-1} and \eqref{RBCD-3-steps-2} admit {multiple closed-form solutions and one of them is} given by
\be\label{RBCD-3-steps-1-sol}
u^{t+1} = u^t + \log(r./\varphi(\zeta(u^t,v^t,U^t)))
\ee
and
\be\label{RBCD-3-steps-2-sol}
v^{t+1} = v^t + \log(c./\kappa(\zeta(u^{t+1},v^t,U^t))),
\ee
where for vectors $a$ and $b$, $a./b$ denotes their component-wise division.  
It is easy to verify that the partial gradient of $g$ with respect to $U$ is given by:
\be\label{grad-g}\nabla_U g(u,v,U) = -\frac{2}{\eta} V_{\pi(u,v,U)} U.\ee
Therefore, \eqref{RBCD-3-steps-3-1}-\eqref{RBCD-3-steps-3-3} give a Riemannian gradient step of $g$ with respect to variable $U$. Also note that \eqref{RBCD-3-steps-3-1} requires to compute $\pi(u^{t+1},v^{t+1},U^t)$, which can be computed using \eqref{def-pi}. The algorithm is terminated when the following stopping criterion is satisfied:
{
\be\label{RBCD-stopping}
\|\xi^{t+1}\|_F\le \frac{\epsilon_1}{4\eta},\quad \|c - \kappa(\zeta(u^{t+1}, v^{t}, U^t))\|_1 \le \frac{\epsilon_2}{8\|C\|_\infty},
\ee
where $\epsilon_1, \epsilon_2$ are pre-given accuracy tolerances.} The reason of using this stopping criterion will be clear in our convergence analysis later.

It should be pointed out that the optimal transportation plan $\pi$ of \eqref{PRW} is not directly computed by RBCD, because the sequence $\pi(u^{t+1},v^{t},U^{t})$ generated in \eqref{RBCD-3-steps} does not satisfy the constraints in \eqref{PRW}. Therefore, a procedure is needed to compute an approximate solution $\pi$ to the original problem \eqref{PRW}. Here we adopt the rounding procedure proposed in \cite{altschuler2017near}, which is outlined in Algorithm \ref{alg:round}, where the notation $a\wedge b$ picks the smaller value between $a$ and $b$. Our final transportation plan is computed by rounding $\pi(u^{t+1},v^t,U^t)$ using the rounding procedure outlined in Algorithm \ref{alg:round}. Combining this rounding procedure and the RBCD outlined above, we arrive at our final algorithm for solving the original PRW problem \eqref{PRW}. Details of our RBCD algorithm for solving \eqref{PRW} are described in Algorithm \ref{alg:RBCD}.

\begin{algorithm}[ht]
\caption{$Round(\pi, r, c)$}
\label{alg:round}
\begin{algorithmic}[1]
\STATE \textbf{Input:} $\pi \in \br^{n\times n}$, $r \in\br^{n}$, $ c \in\br^{n}$.
\STATE $X = \Diag(x)$ with $x_i = \frac{r_i}{\varphi(\pi)_i} \wedge 1$
\STATE $\pi' = X\pi$
\STATE $Y = \Diag(y)$ with $y_j = \frac{c_j}{\kappa(\pi')_j} \wedge 1$
\STATE $\pi'' = \pi'Y$
\STATE $err_r = r - \varphi(\pi''), err_c = c - \kappa(\pi'')$
\STATE \textbf{Output:} $\pi'' + err_rerr_c^\top / \|err_r\|_1$.
\end{algorithmic}
 \end{algorithm}

\begin{algorithm}[ht]
\caption{Riemannian Block Coordinate Descent Algorithm (RBCD) }
\label{alg:RBCD}
\begin{algorithmic}[1]
\STATE \textbf{Input:} $\{(x_i,r_i)\}_{i\in [n]}$ and $\{(y_j,c_j)\}_{j\in [n]}$, $U^0\in\M$, $u^0, v^0\in\br^n$, and accuracy tolerance $\epsilon_1\geq\epsilon_2>0$. Set parameters ($L_1$ and $L_2$ are defined in Proposition \ref{pro:retr})
\be\label{alg:RBCD-param}\tau = \frac{1}{4  L_2 \|C\|_\infty /\eta + \rho L_1^2}, \quad \eta=\frac{\epsilon_2}{4\log(n)+2}.\ee
\FOR{$t = 0, 1, 2, \ldots,$}
\STATE Compute $u^{t+1}$ by \eqref{RBCD-3-steps-1-sol}
\STATE Compute $v^{t+1}$ by \eqref{RBCD-3-steps-2-sol}
\STATE Compute $V_{\pi(u^{t+1},v^{t+1},U^t)}$, $\xi^{t+1}$ and $U^{t+1}$ by \eqref{RBCD-3-steps-3-1}-\eqref{RBCD-3-steps-3-3}
\IF{\eqref{RBCD-stopping} is satisfied}
\STATE break
\ENDIF
\ENDFOR
\STATE \textbf{Output:} $\hat{u} = u^{t+1}$, $\hat{v} = v^t$,  $\hat{U} = U^t$, and $\hat{\pi}= Round(\pi(\hat{u}, \hat{v}, \hat{U}), r, c)$.
\end{algorithmic}
\end{algorithm}

\begin{remark}
We remark that \eqref{RBCD-3-steps-1-sol} and \eqref{RBCD-3-steps-2-sol} are the steps in the Sinkhorn's algorithm \cite{cuturi2013sinkhorn}. It is easy to verify the following identities for any $t\geq 0$:
{
\be\label{sinkhorn-update-satisfy-equation}
\varphi(\zeta(u^{t+1},v^t,U^t)) = r, \qquad \kappa(\zeta(u^{t+1},v^{t+1},U^t)) = c,
\ee
and
\be\label{sinkhorn-update-satisfy-L1}
\|\zeta(u^{t+1},v^t,U^t)\|_1 = \|\zeta(u^{t+1},v^{t+1},U^t)\|_1 = 1.
\ee
Therefore, we naturally have
\be\bad\label{sinkhorn-update-satisfy-equation2}
\pi(u^{t+1},v^t,U^t) =\zeta(u^{t+1},v^t,U^t), \ \pi(u^{t+1},v^{t+1},U^t) = \zeta(u^{t+1},v^{t+1},U^t).
\ead\ee
}
\end{remark}

\section{Convergence Analysis}\label{sec:conv}

In this section, we show that $(\hat{\pi},\hat{U})$ returned by Algorithm \ref{alg:RBCD} is an $\epsilon$-stationary point of the PRW problem \eqref{PRW}. We will also analyze its iteration complexity and complexity of arithmetic operations for obtaining such a point. The $\epsilon$-stationary point for problem \eqref{PRW} is defined as follows.
\begin{definition} \label{def:primalsta}
We call $(\hat{\pi},\hat{U})\in \Pi(\mu, \nu)\times\mathcal{M}$ an $(\epsilon_1,\epsilon_2)$-stationary point of the PRW problem \eqref{PRW}, if the following two inequalities hold:
\begin{subequations}\label{def:primalsta-eq}
\begin{align}
\|\emph{grad}_Uf(\hat{\pi}, \hat{U})\|_F \  & \leq \epsilon_1, \label{def:primalsta-eq-1} \\
f(\hat{\pi}, \hat{U}) -  \min_{\pi \in \Pi(\mu, \nu)} f(\pi, \hat{U})  & \leq \epsilon_2.\label{def:primalsta-eq-2}
\end{align}
\end{subequations}
\end{definition}
\begin{remark}
In \cite{lin2020projection}, the authors defined the $\epsilon$-stationary point of PRW \eqref{PRW} as the pair $(\hat{\pi},\hat{U})$ that satisfies:
\begin{subequations}\label{def:eps-lin}
\begin{align}
\dist(0,\subdiff f(\hat{U})) \leq \epsilon \label{def:eps-lin-1}\\
f(\hat{\pi}, \hat{U}) -  \min_{\pi \in \Pi(\mu, \nu)} f(\pi, \hat{U})  & \leq \epsilon, \label{def:eps-lin-2}
\end{align}
\end{subequations}
where $\subdiff$ denotes the Riemannian subgradient, and
\[f(U) := \min_{\pi\in\Pi(\mu_n,\nu_n)} f(\pi,U) :=\sum_{i,j=1}^n\pi_{ij}\|U^\top x_i-U^\top y_j\|^2.\]
In the appendix, we will show that our \eqref{def:primalsta-eq} implies \eqref{def:eps-lin} when $\epsilon_1=\epsilon_2=\epsilon$. Therefore, it is harder to satisfy the conditions in our Definition \eqref{def:primalsta}.
\end{remark}

We now introduce two useful lemmas that will be used in our subsequent analysis.

\begin{lemma}[\cite{cuturi2013sinkhorn}]\label{lem:unique}
Given a cost matrix $M \in \br^{n \times n}$ and $r, c \in \Delta^n$, the entropy-regularized OT problem \eqref{OT-reg} has a unique minimizer with the form {$\zeta = XAY$, where $A = \exp{(- M/\eta)}$ and $X, Y \in \br^{n \times n}_+$ are diagonal matrices.} The matrices $X$ and $Y$ are unique up to a constant factor.
\end{lemma}

{Now from the defintion of $\zeta(u,v,U)$ in \eqref{def-zeta} and the, we know that $\zeta(u,v,U)$ is in the form given in Lemma \ref{lem:unique} with $X = \Diag(u)$, $Y = \Diag(v)$ and $A = \exp{(-M/\eta)}$ with $M_{ij}:= \|U^\top(x_i-y_j)\|^2$.
Therefore, for fixed $(u,v,U)$, by denoting $r' := \varphi(\zeta(u,v,U))$ and $c' := \kappa(\zeta(u,v,U))$, from Lemma \ref{lem:unique} we know that $\zeta(u,v,U)$} is the unique optimal solution of the following regularized OT problem:
\[
\min_{\pi \in \br_+^{n\times n}} \langle \pi, M \rangle - \eta H(\pi), \ \st, \ \varphi(\pi)=r', \kappa(\pi) = c'.
\]
As a result, we arrive at the following definition of stationary point for problem  \eqref{eq:dualprw}.
\begin{definition} \label{lem:dualstationary}
We say $(u^*,v^*,U^*)$ is a stationary point of the problem \eqref{eq:dualprw} if the following equalities hold:
{
\be\label{lem:dualstationary-eq-1}
\varphi(\zeta{(u^*, v^*, U^*)}) = r, \quad \kappa(\zeta{(u^*, v^*, U^*)}) = c, \quad \emph{grad}_U g(u^*, v^*, U^*) = 0,
\ee}
where $\zeta{(u, v, U)}$ is defined in \eqref{def-zeta}.
\end{definition}

\begin{lemma}(\cite[Lemma 7]{altschuler2017near})\label{lem:roundcloseness}
Let $r, c \in \Delta^n$, $\pi\in\br_+^{n\times n}$ and $\hat{\pi}$ be the output of $Round(\pi, r,c)$. The following inequality holds:
\[
\|\hat{\pi} - \pi\|_1 \le 2(\|\varphi(\pi) - r\|_1 + \|\kappa(\pi) - c\|_1).
\]
\end{lemma}

The next lemma shows that $(\hat{\pi},\hat{U})$ returned by Algorithm \ref{alg:RBCD} is an {$(\epsilon_1, \epsilon_2)$}-stationary point of problem \eqref{PRW} as defined in Definition \ref{def:primalsta}.

\begin{lemma} \label{lem:dualtoprimal}
Assume Algorithm \ref{alg:RBCD} terminates at the $T$-th iteration. Set $\epsilon_1 \ge \epsilon_2$. Then $(\hat{\pi},\hat{U})$ returned by Algorithm \ref{alg:RBCD}, i.e., $\hat{\pi} = Round(\pi(u^{T+1},v^T,U^T),r,c)$ and $\hat{U} := U^T$, is an  {$(\epsilon_1, \epsilon_2)$}-stationary point of problem \eqref{PRW} as defined in Definition \ref{def:primalsta}.
\end{lemma}
\begin{proof}
{
When Algorithm \ref{alg:RBCD} terminates at the $T$-th iteration, according to \eqref{RBCD-stopping}, we have
\be\label{stop-kappa}
\|\kappa(\pi{(u^{T+1}, v^T, U^T)}) - c\|_1\le \frac{\epsilon_2}{8\|C\|_\infty},
\ee
and
\be\label{stop-U}
\eta\|\grad_U g(u^{T+1},v^{T+1},U^T)\|_F \leq \frac{\epsilon_1}{4}.
\ee
Denote $\bar{\pi} = \pi{(u^{T+1}, v^T, U^T)}$, $\pi^* = \argmin_{\pi\in \Pi(\mu,\nu)} f(\pi, U^T)$, $r' = \varphi(\bar{\pi})$, $c' = \kappa(\bar{\pi})$, and $\pi' = Round(\pi^*, r',c')$. Note that $r'=r$ from \eqref{sinkhorn-update-satisfy-equation}. Lemma \ref{lem:roundcloseness} implies that
\be\label{diff-pi'-pi*}
\|\pi' - \pi^*\|_1 \le 2(\|\varphi(\pi^*) - r'\|_1 + \|\kappa(\pi^*) - c'\|_1) = 2(\|r - r'\|_1 + \|c - c'\|_1) = 2\|c - c'\|_1.
\ee
By Lemma \ref{lem:unique}, $\bar{\pi}$ is the optimal solution to
\be \label{eq:optsol}
\min_{\pi\in \Pi(\mu_n', \nu_n')} \langle \pi, M \rangle - \eta H(\pi),
\ee
where $M_{i,j} = \|{(U^T)}^\top(x_i - y_j)\|_2^2$, $\mu_n' = \sum_{i=1}^n r_i'\delta_{x_i}$ and $\nu_n' = \sum_{j=1}^n c_j'\delta_{y_j}$. Therefore, we have
\be\label{lem:sinkhornconvergence-proof-eq-1}
 \langle \bar{\pi}, M \rangle - \eta H(\bar{\pi}) \le  \langle \pi', M \rangle - \eta H(\pi'),
\ee
which implies,
\begin{align}\label{bound-pibar-pi*}
 \langle \bar{\pi}, M \rangle -  \langle \pi^*, M \rangle &=  \langle \bar{\pi}, M \rangle -  \langle \pi', M \rangle+ \langle \pi', M \rangle - \langle \pi^*, M \rangle \\
&\le \eta (H(\bar{\pi}) - H(\pi')) + 2\|c - c'\|_1\|M\|_\infty \nonumber\\
&\le \eta (2\log(n)+1) + 2\|c - c'\|_1\|M\|_\infty,\nonumber
\end{align}
where the first inequality is due to \eqref{lem:sinkhornconvergence-proof-eq-1}, \eqref{diff-pi'-pi*} and the H\"{o}lder's inequality, and the second inequality is due to the fact that $0 \le H(\bar{\pi}) \le 2\log(n)+1$ and $0\leq H(\pi') \le 2\log(n)+1$. Since Lemma \ref{lem:roundcloseness} also implies
\be\label{diff-pihat-pibar}
\|\hat{\pi} - \bar{\pi}\|_1 \le 2(\|\varphi(\bar{\pi}) - r\|_1 + \|\kappa(\bar{\pi}) - c\|_1) = 2(\|r - r'\|_1 + \|c - c'\|_1) = 2\|c - c'\|_1,
\ee
we then have
\begin{align*}
 \langle \hat{\pi}, M \rangle -  \langle \pi^*, M \rangle & = \langle \hat{\pi}-\bar{\pi},M \rangle + \langle \bar{\pi}-\pi^*,M \rangle \\ &\le \eta (2\log(n)+1) + 4 \|c - c'\|_1\|M\|_\infty\\
 & \le \eta (2\log(n)+1) + 4\|c - c'\|_1\|C\|_\infty\\
 & \le \eta (2\log(n)+1)  + \epsilon_2 / 2,
\end{align*}
where the first inequality is due to \eqref{bound-pibar-pi*}, \eqref{diff-pihat-pibar} and the H\"{o}lder's inequality, and the last inequality follows from \eqref{stop-kappa}.
By choosing $\eta = {\epsilon_2}/{(4\log(n)+2)}$ as in \eqref{alg:RBCD-param}, we obtain
\be\label{eq:fhatpi}
f(\hat{\pi}, \hat{U}) \le \min_{\pi \in \Pi(\mu,\nu)} f(\pi, \hat{U})+ \epsilon_2.
\ee
We now bound $\|\text{grad}_U f(\hat{\pi}, U^T)\|_F$. For simplicity of the notation, we further denote $\tilde{\pi} = \pi{(u^{T+1}, v^{T+1}, U^T)}$.
Since
\[
\nabla_U f(\tilde{\pi}, U^T) = 2V_{\tilde{\pi}}U^T = -\eta \nabla_U g(u^{T+1}, v^{T+1}, U^T),
\]
by combining \eqref{stop-U}, we know that,
\[
\|\text{grad}_U f(\tilde{\pi}, U^T)\|_F = \|\proj_{\Tg_{U^T}\M} (\nabla_U f(\tilde{\pi}, U^T))\|_F \le \epsilon_1/4.
\]
Therefore,
\begin{align}\label{lem:dualtoprimal-proof-eq-1}
\|\text{grad}_U f(\hat{\pi}, U^T)\|_F &= \|\proj_{\Tg_{U^T}\M} (2 V_{\hat{\pi}}U^T)\|_F\\
& = \|\proj_{\Tg_{U^T}\M} (2 (V_{\hat{\pi}} -  V_{\bar{\pi}} + V_{\bar{\pi}} - V_{\tilde{\pi}} + V_{\tilde{\pi}})U^T)\|_F\nonumber\\
& \le  2\|(V_{\hat{\pi}} -  V_{\bar{\pi}})U^T\|_F + 2\|(V_{\bar{\pi}} - V_{\tilde{\pi}})U^T\|_F + \|\proj_{\Tg_{U^T}\M} (2V_{\tilde{\pi}}U^T)\|_F\nonumber\\
& \le  2\|V_{\hat{\pi}} -  V_{\bar{\pi}}\|_F + 2\|V_{\bar{\pi}} - V_{\tilde{\pi}}\|_F + \|\text{grad}_U f(\tilde{\pi}, U^T)\|_F\nonumber\\
& \le  2\|V_{\hat{\pi}} -  V_{\bar{\pi}}\|_F + 2\|V_{\bar{\pi}} - V_{\tilde{\pi}}\|_F + \epsilon_1/4 \nonumber.
\end{align}
In the next we will bound $\|V_{\hat{\pi}} -  V_{\bar{\pi}}\|_F$ and $\|V_{\bar{\pi}} - V_{\tilde{\pi}}\|_F$.
By combining \eqref{diff-pihat-pibar} and \eqref{stop-kappa}, we have,
\be\label{lem:dualtoprimal-proof-eq-2}
2\|V_{\hat{\pi}} -  V_{\bar{\pi}}\|_F \le 2\|C\|_\infty \|\hat{\pi} - \bar{\pi}\|_1 \le 4 \|C\|_\infty \|\kappa(\bar{\pi}) - c\|_1 \le \frac{\epsilon_2}{2} \le \frac{\epsilon_1}{2}.
\ee
Moreover, by \eqref{sinkhorn-update-satisfy-equation2} we have
{
\begin{align*}
   & \|\bar{\pi} -\tilde{\pi} \|_1 \\
= &  \|\zeta(u^{T+1}, v^{T}, U^T)  - \zeta(u^{T+1}, v^{T+1}, U^T) \|_1 \\
= & \sum_{ij} \left\lvert \exp{\left(-\frac{1}{\eta}\|(U^T)^\top(x_i - y_j)\|^2 + u^{T+1}_i + v^{T}_j\right)} -  \exp{\left(-\frac{1}{\eta}\|(U^T)^\top(x_i - y_j)\|^2 + u^{T+1}_i + v^{T+1}_j\right)} \right\rvert\\
= & \sum_{ij} \exp{\left(-\frac{1}{\eta}\|(U^T)^\top(x_i - y_j)\|^2 + u^{T+1}_i + v^{T}_j\right)} \lvert 1 -  \exp{(v^{T+1}_j - v^{T}_j)} \rvert\\
= & \sum_j [\kappa(\bar{\pi})]_j \left\lvert 1 -  \frac{c_j}{[\kappa(\bar{\pi})]_j }\right\rvert = \|\kappa(\bar{\pi}) -  c\|_1,
\end{align*}}
where the last equality is due to \eqref{def-pi} and \eqref{RBCD-3-steps-2-sol}.
Therefore, from \eqref{stop-kappa} we have
\be\label{lem:dualtoprimal-proof-eq-3}
2\|V_{\bar{\pi}} - V_{\tilde{\pi}}\|_F \le 2 \|C\|_\infty \|\bar{\pi} -\tilde{\pi} \|_1 = 2 \|C\|_\infty\|\kappa(\bar{\pi}) -  c\|_1 \le \frac{\epsilon_2}{4} \le \frac{\epsilon_1}{4}.
\ee
Combining \eqref{lem:dualtoprimal-proof-eq-1}, \eqref{lem:dualtoprimal-proof-eq-2} and \eqref{lem:dualtoprimal-proof-eq-3} gives
\be\label{eq:gradfhatpi}
\|\text{grad}_U f(\hat{\pi}, \hat{U})\|_F \le \epsilon_1.
\ee
Combining \eqref{eq:fhatpi} and \eqref{eq:gradfhatpi} implies that $(\hat{\pi},\hat{U})$ is a $(\epsilon_1, \epsilon_2)$-stationary point of \eqref{PRW} as defined in Definition \ref{def:primalsta}. }

\end{proof}

The rest of this section is devoted to analyzing the iteration complexity of Algorithm \ref{alg:RBCD}.
To this end, we need to show that the function $g$ is lower bounded, and monotonically decreases along the course of RBCD (Algorithm \ref{alg:RBCD}). These results are proved in the following lemmas.

\begin{lemma}[Lower boundedness of $g$]\label{lem:gbound}
Denote $(u^*,v^*,U^*)$ as the global minimum of $g$ in \eqref{eq:dualprw}. The following inequality holds:
{
\be\label{eq:gbound}
g(u^*, v^*, U^*) \ge - \frac{\|C\|_\infty}{\eta}.
\ee}
\end{lemma}
\begin{proof}
From \eqref{lem:dualstationary-eq-1} we have
{
\be\label{lem:gbound-eq-1}
\|\zeta{(u^*, v^*, U^*)}\|_1 = \sum_{i,j} \exp{\left(-\frac{1}{\eta}\|(U^*)^\top(x_i - y_j)\|^2 + u^*_i + v^*_j\right)} = 1,
\ee}
{which implies that $\zeta{(u^*, v^*, U^*)} = \pi{(u^*, v^*, U^*)}$ and
\be\label{eq:geq}
g(u^*, v^*, U^*) = \log\left(\sum_{i,j} [\pi{(u^*, v^*, U^*)}]_{ij} \right) - \langle u^*, r \rangle - \langle v^*, c \rangle = - \langle u^*, r \rangle - \langle v^*, c \rangle.
\ee}
Notice that $\|C\|_\infty \geq \|(U^*)^\top(x_i - y_j)\|^2$ for any $i,j$, together with \eqref{lem:gbound-eq-1} we have
\[
\exp{\left(-\frac{1}{\eta}\|C\|_\infty + u^*_i + v^*_j\right)} \le \exp{\left(-\frac{1}{\eta}\|(U^*)^\top(x_i - y_j)\|^2 + u^*_i + v^*_j\right)} \le 1, \forall i, j,
\]
which further implies
\be\label{lem:gbound-eq-2}
u^*_i + v^*_j  \le \frac{1}{\eta}\|C\|_\infty, \forall i, j.
\ee
Since $r,c\in\Delta^n$, \eqref{lem:gbound-eq-2} indicates that
\[
\langle u^*, r \rangle + \langle v^*, c \rangle\ \le \frac{1}{\eta}\|C\|_\infty,
\]
which, combining with \eqref{eq:geq}, yields the desired result.
\end{proof}

In the next a few lemmas, we show that function $g$ has a sufficient decrease when $u$, $v$, and $U$ are updated in RBCD. We need the following lemma first.

\begin{lemma}\label{lem:lipschitz}
Let $\{(u^t, v^t, U^t)\}$ be the sequence generated by Algorithm \ref{alg:RBCD}. For any $\alpha \in [0,1]$, the following inequality holds for any $U\in\M$:
\[
\|\nabla_U g(u^{t+1}, v^{t+1}, U^t) - \nabla_U g(u^{t+1}, v^{t+1}, \alpha U + (1-\alpha) U^t)\|_F \le \rho \alpha \|U^t - U\|_F,
\]
where $\rho = \frac{2\|C\|_\infty}{\eta} + \frac{4\|C\|^2_\infty}{\eta^2} $.
\end{lemma}
\begin{proof}
Denote $U^\alpha = \alpha U + (1-\alpha) U^t$. Note that $U^\alpha$ is not necessary on $\M$, though $U, U^t \in \M$.
From \eqref{grad-g} we have,
\begin{align}\label{lem:lipschitz-proof-eq-1}
&\|\nabla_U g(u^{t+1}, v^{t+1}, U^t) - \nabla_U g(u^{t+1}, v^{t+1}, U^\alpha)\|_F\\
= & \frac{2}{\eta} \| V_{\pi{(u^{t+1}, v^{t+1}, U^t)}} U^t -  V_{\pi{(u^{t+1}, v^{t+1}, U^\alpha)}} U^\alpha\|_F \nonumber\\
= & \frac{2}{\eta} \| V_{\pi{(u^{t+1}, v^{t+1}, U^t)}} U^t - V_{\pi{(u^{t+1}, v^{t+1}, U^t)}} U^\alpha + V_{\pi{(u^{t+1}, v^{t+1}, U^t)}} U^\alpha -  V_{\pi{(u^{t+1}, v^{t+1}, U^\alpha)}} U^\alpha\|_F\nonumber\\
\le & \frac{2}{\eta} \| V_{\pi{(u^{t+1}, v^{t+1}, U^t)}} (U^t - U^\alpha)\|_F + \frac{2}{\eta}\|(V_{\pi{(u^{t+1}, v^{t+1}, U^t)}} -  V_{\pi{(u^{t+1}, v^{t+1}, U^\alpha)}}) U^\alpha\|_F\nonumber\\
\le & \frac{2}{\eta} \| V_{\pi{(u^{t+1}, v^{t+1}, U^t)}} (U^t - U^\alpha)\|_F + \frac{2\alpha}{\eta}\|(V_{\pi{(u^{t+1}, v^{t+1}, U^t)}} -  V_{\pi{(u^{t+1}, v^{t+1}, U^\alpha)}}) U\|_F\nonumber\\
& +  \frac{2(1-\alpha)}{\eta}\|(V_{\pi{(u^{t+1}, v^{t+1}, U^t)}} -  V_{\pi{(u^{t+1}, v^{t+1}, U^\alpha)}}) U^t\|_F \nonumber\\
\le & \frac{2}{\eta} \| V_{\pi{(u^{t+1}, v^{t+1}, U^t)}}\|_F \|U^t- U^\alpha\|_F + \frac{2}{\eta}\|V_{\pi{(u^{t+1}, v^{t+1}, U^t)}}  -  V_{\pi{(u^{t+1}, v^{t+1}, U^\alpha)}} \|_F,\nonumber
\end{align}
Since $\|\pi{(u^{t+1},v^{t+1},U^t)}\|_1 = 1$, we have
\be\label{eq:Vbound}
\|V_{\pi{(u^{t+1},v^{t+1},U^t)}}\|_F \le \sum_{i,j} [\pi{(u^{t+1},v^{t+1},U^t)}]_{ij} \|(x_i - y_j)(x_i - y_j)^\top\|_F \le \max_{ij} |x_i - y_j|^2 = \|C\|_\infty.
\ee
Note that for fixed $U$, the objective function $f_\eta(\pi,U)$ in \eqref{PRW-reg} is $\eta$-strongly convex with respect to $\pi$ under the $\ell_1$ norm metric, which implies
{
\begin{align}
&f_\eta(\pi(u^{t+1},v^{t+1},U^t), U^\alpha) \ge  f_\eta( \pi(u^{t+1},v^{t+1},U^\alpha), U^\alpha) + \langle \nabla_\pi f_\eta(\pi(u^{t+1},v^{t+1},U^\alpha), U^\alpha), \nonumber\\
&\pi(u^{t+1},v^{t+1},U^t) - \pi(u^{t+1},v^{t+1},U^\alpha)\rangle + \frac{\eta}{2}\|\pi(u^{t+1},v^{t+1},U^t) - \pi(u^{t+1},v^{t+1},U^\alpha)\|_1^2\\
&f_\eta(\pi(u^{t+1},v^{t+1},U^\alpha), U^\alpha) \ge  f_\eta(\pi(u^{t+1},v^{t+1},U^t), U^\alpha) + \langle \nabla_\pi f_\eta(\pi(u^{t+1},v^{t+1},U^t), U^\alpha),\nonumber \\
& \pi(u^{t+1},v^{t+1},U^\alpha) - \pi(u^{t+1},v^{t+1},U^t) \rangle + \frac{\eta}{2}\|\pi(u^{t+1},v^{t+1},U^t)- \pi(u^{t+1},v^{t+1},U^\alpha)\|_1^2.
\end{align}
By adding the above two inequalities, we have}
\begin{align}
& \langle \nabla_\pi f_\eta(\pi(u^{t+1},v^{t+1},U^\alpha),U^\alpha) - \nabla_\pi f_\eta(\pi(u^{t+1},v^{t+1},U^t),U^\alpha) , \pi(u^{t+1},v^{t+1},U^\alpha) - \pi(u^{t+1},v^{t+1},U^t)\rangle \nonumber \\
\ge & \eta\|\pi(u^{t+1},v^{t+1},U^\alpha) - \pi(u^{t+1},v^{t+1},U^t)\|_1^2. \label{eq:primalstronglysum}
\end{align}
{
Moreover, note that
\be\bad\label{grad-pi-f-eta}
(\nabla_\pi f_\eta(\pi,U))_{ij} =& \|U^\top(x_i - y_j)\|^2 + \eta\log(\pi_{ij}),
\ead\ee
which, combining with \eqref{def-zeta}  and \eqref{def-pi}, yields
\begin{align*}
&[\nabla_\pi f_\eta(\pi(u,v,U),U)]_{ij} \\
= & \|U^\top(x_i - y_j)\|^2 + \eta \log([\pi(u,v,U)]_{ij})\\
= & \|U^\top(x_i - y_j)\|^2 + \eta\left(-\frac{1}{\eta}\|U^\top(x_i - y_j)\|^2 + u_{i} + v_{j} \right)- \eta \log(\|\zeta(u,v,U)\|_1)\\
= & \eta(u_{i} + v_{j})- \eta \log(\|\zeta(u,v,U)\|_1).
\end{align*}
We further have
\be\bad\label{eq:primalstronglysum3}
& \langle \nabla_\pi f_\eta(\pi(u^{t+1},v^{t+1},U^t),U^t) - \nabla_\pi f_\eta(\pi(u^{t+1},v^{t+1},U^\alpha),U^\alpha) , \pi(u^{t+1},v^{t+1},U^\alpha) - \pi(u^{t+1},v^{t+1},U^t)\rangle \\
= & \sum_{i,j}  \left(- \eta \log(\|\zeta(u^{t+1},v^{t+1},U^t)\|_1) + \eta \log(\|\zeta(u^{t+1},v^{t+1},U^\alpha)\|_1) \right) \left(\pi(u^{t+1},v^{t+1},U^\alpha)_{i,j} - \pi(u^{t+1},v^{t+1},U^t)_{i,j}\right)\\
=& \left(- \eta \log(\|\zeta(u^{t+1},v^{t+1},U^t)\|_1) + \eta \log(\|\zeta(u^{t+1},v^{t+1},U^\alpha)\|_1) \right) \sum_{i,j}   \left(\pi(u^{t+1},v^{t+1},U^\alpha)_{i,j} - \pi(u^{t+1},v^{t+1},U^t)_{i,j}\right)\\
=& 0.
\ead\ee
Summing \eqref{eq:primalstronglysum} and \eqref{eq:primalstronglysum3} leads to}
\begin{align}
& \langle \nabla_\pi f_\eta(\pi(u^{t+1},v^{t+1},U^t),U^t) - \nabla_\pi f_\eta(\pi(u^{t+1},v^{t+1},U^t),U^\alpha) , \pi(u^{t+1},v^{t+1},U^\alpha) - \pi(u^{t+1},v^{t+1},U^t)\rangle \nonumber \\
\ge & \eta\|\pi(u^{t+1},v^{t+1},U^\alpha) - \pi(u^{t+1},v^{t+1},U^t)\|_1^2, \label{eq:primalstronglysum2}
\end{align}
which, by H\"{o}lder's inequality, yields,
\begin{align}\label{eq:inftynorm}
    & \eta\|\pi(u^{t+1},v^{t+1},U^\alpha) - \pi(u^{t+1},v^{t+1},U^t)\|_1 \\
\le & \|\nabla_\pi f_\eta(\pi(u^{t+1},v^{t+1},U^t),U^t) - \nabla_\pi f_\eta( \pi(u^{t+1},v^{t+1},U^t),U^\alpha)\|_\infty \nonumber\\
\le & \max_{i,j}\  \lvert \|(U^\alpha)^\top(x_i - y_j)\|_2^2 - \|(U^t)^\top(x_i - y_j)\|_2^2  \rvert\nonumber\\
=  & \max_{i,j}\  \lvert (x_i - y_j)^\top(U^\alpha (U^\alpha)^\top - U^t(U^t)^\top) (x_i - y_j) \rvert \nonumber\\
\le & (\max_{i,j}\  \| x_i - y_j\|^2) \|U^\alpha (U^\alpha)^\top - U^t(U^t)^\top\|_F\nonumber\\
=  & \|C\|_\infty\|U^\alpha (U^\alpha)^\top - U^t(U^t)^\top\|_F,\nonumber
\end{align}
where the second inequality follows from \eqref{grad-pi-f-eta}.
Furthermore, since $U, U^t\in\M$, we have
\be \label{eq:UUTdiff}\bad
   & \|U^\alpha (U^\alpha)^\top - U^t(U^t)^\top\|_F \\
= & \|U^\alpha (U^\alpha)^\top - U^t(U^\alpha)^\top + U^t(U^\alpha)^\top - U^t(U^t)^\top\|_F\\
\le & \|(U^\alpha - U^t)(U^\alpha)^\top\|_F + \|U^t(U^\alpha - U^t)^\top\|_F\\
\le & \|(U^\alpha - U^t)(\alpha U + (1-\alpha) U^t)^\top\|_F + \|U^\alpha - U^t\|_F\\
\le & \alpha\|(U^\alpha - U^t) U^\top\|_F + (1-\alpha) \|(U^\alpha - U^t) (U^t)^\top\|_F + \|(U^\alpha - U^t)\|_F\\
=  & 2\|U^\alpha - U^t\|_F.
\ead\ee
By combining \eqref{eq:inftynorm} and \eqref{eq:UUTdiff}, we have
\begin{align}\label{eq:bound-V-pi-diff}
& \|V_{\pi{(u^{t+1},v^{t+1},U^t)}}  -  V_{\pi{(u^{t+1},v^{t+1},U^\alpha)}} \|_F\\
\leq & \|C\|_\infty \|\pi{(u^{t+1},v^{t+1},U^t)} - \pi{(u^{t+1},v^{t+1},U^\alpha)}\|_1 \nonumber\\
\leq & \frac{2\|C\|^2_\infty}{\eta}\|U^t - U^\alpha \|_F.\nonumber
\end{align}
Plugging \eqref{eq:Vbound} and \eqref{eq:bound-V-pi-diff} into \eqref{lem:lipschitz-proof-eq-1}  yields the desired result.
\end{proof}

Now we are ready to show the sufficient decrease of $g$ when $U$ is updated.
\begin{lemma}[Decrease of $g$ in $U$]\label{lem:decU}
Let $(u^t, v^t, U^t)$ be the sequence generated by Algorithm \ref{alg:RBCD}. For any $t\geq 0$, the following inequality holds:
\be\label{decrease-U}
g(u^{t+1}, v^{t+1}, U^{t+1}) - g(u^{t+1}, v^{t+1}, U^t ) \le - \frac{1}{8L_2\|C\|_\infty/\eta + 2\rho L_1^2 } \|\xi^{t+1}\|^2_F,
\ee
where $\rho$ is defined in Lemma \ref{lem:lipschitz}, and $L_1$ and $L_2$ are defined in Proposition \ref{pro:retr}.
\end{lemma}

\begin{proof}
By setting $U = U^{t+1}$ in Lemma \ref{lem:lipschitz}, we have,
\be\label{eq:lipineq} \bad
&\lvert g(u^{t+1}, v^{t+1}, U^{t+1}) - g(u^{t+1}, v^{t+1}, U^t) - \langle \nabla_U g(u^{t+1}, v^{t+1}, U^t), U^{t+1} - U^t \rangle \rvert\\
= & \left\lvert \int_0^1 \langle \nabla_U g(u^{t+1}, v^{t+1}, \alpha U^{t+1} + (1-\alpha) U^{t}) - \nabla_U g(u^{t+1}, v^{t+1}, U^t), U^{t+1} - U^t \rangle d\alpha  \right\rvert\\
\le & \int_0^1 \| \nabla_U g(u^{t+1}, v^{t+1}, \alpha U^{t+1} + (1-\alpha) U^{t}) - \nabla_U g(u^{t+1}, v^{t+1}, U^t)\|_F \|U^{t+1} - U^t\|_F d\alpha\\
\le & \int_0^1 \rho \alpha \|U^{t+1} - U^t\|_F^2 d\alpha\\
= & \frac{\rho}{2} \|U^{t+1} - U^t\|_F^2 \\
= & \frac{\rho}{2}\|\retr_{U^t}(- \tau \xi^{t+1}) - U^t\|_F^2 \\
\le & \frac{\rho\tau^2L_1^2}{2}\|\xi^{t+1}\|_F^2,
\ead\ee
where
where the last inequality follows from Proposition \ref{pro:retr}.
Moreover, we have
\be\label{eq:lipinnerterm}\bad
   & \langle \nabla_U g(u^{t+1}, v^{t+1}, U^t), U^{t+1} - U^t \rangle \\
= & \langle \nabla_U g(u^{t+1}, v^{t+1}, U^t),  \retr_{U^t}(- \tau \xi^{t+1}) - U^t \rangle \\
= & \langle \nabla_U g(u^{t+1}, v^{t+1}, U^t),  - \tau \xi^{t+1} \rangle +  \langle \nabla_U g(u^{t+1}, v^{t+1}, U^t),  \retr_{U^t}(- \tau \xi^{t+1}) - (U^t - \tau \xi^{t+1}) \rangle \\
\le & - \tau \langle \nabla_U g(u^{t+1}, v^{t+1}, U^t),  \xi^{t+1} \rangle +  \| \nabla_U g(u^{t+1}, v^{t+1}, U^t)\|_F\| \retr_{U^t}(- \tau \xi^{t+1}) - (U^t - \tau \xi^{t+1} )\|_F\\
\le & - \tau \langle \nabla_U g(u^{t+1}, v^{t+1}, U^t),  \xi^{t+1} \rangle +  \frac{2}{\eta} \|V_{\pi{(u^{t+1},v^{t+1},U^t)}}U^t\|_F \times L_2 \tau^2 \|\xi^{t+1}\|_F^2\\
\le &  - \tau \langle \nabla_U g(u^{t+1}, v^{t+1}, U^t),  \xi^{t+1} \rangle +  \frac{2}{\eta}  L_2 \tau^2 \|C\|_\infty \|\xi^{t+1}\|_F^2\\
=  & - \tau \| \xi^{t+1}\|_F^2 +  \frac{2}{\eta}  L_2 \tau^2 \|C\|_\infty \|\xi^{t+1}\|_F^2,
\ead\ee
where the second inequality follows from Proposition \ref{pro:retr}.
Combining \eqref{eq:lipineq} and \eqref{eq:lipinnerterm} yields,
\[
g(u^{t+1}, v^{t+1}, U^{t+1}) - g(u^{t+1}, v^{t+1}, U^t )
\leq  -\tau\left(1 - \left(\frac{2}{\eta}  L_2 \|C\|_\infty  + \frac{\rho}{2}L_1^2\right) \tau\right)\|\xi^{t+1}\|_F^2.
\]
Finally, choosing $\tau = \frac{1}{4  L_2 \|C\|_\infty/\eta  + \rho L_1^2}$ as in \eqref{alg:RBCD-param} gives the desired result \eqref{decrease-U}.
\end{proof}

Now we show the sufficient decrease of $g$ when $u$ is updated.
\begin{lemma}[Decrease of $g$ in $u$]\label{lem:decinu}
Let $\{(u^t, v^t, U^t)\}$ be the sequence generated by Algorithm \ref{alg:RBCD}. For any $t \ge 0$, the following inequality holds
{
\[
g(u^{t+1}, v^t, U^t) - g(u^{t}, v^t, U^t ) \ \le 0.
\]}
\end{lemma}
\begin{proof}
{It is a result of \eqref{RBCD-3-steps-1}.}
\end{proof}

Now we show the sufficient decrease of $g$ when $v$ is updated, and its proof largely follows \cite[Theorem 1]{altschuler2017near}.
\begin{lemma}[Decrease of $g$ in $v$]\label{lem:decinv}
Let $(u^t, v^t, U^t)$ be the sequence generated by Algorithm \ref{alg:RBCD}. For any $t \ge 0$, the following inequality holds:
\[
g(u^{t+1}, v^{t+1}, U^t) - g(u^{t+1}, v^t, U^t ) \le - \frac{1}{2} \|\kappa(\pi{(u^{t+1}, v^t, U^t)}) - c\|_1^2.
\]
\end{lemma}

\begin{proof}
From \eqref{sinkhorn-update-satisfy-L1} and \eqref{sinkhorn-update-satisfy-equation2} we have, {
\begin{align*}
   &g(u^{t+1}, v^{t+1}, U^t) - g(u^{t+1}, v^t, U^t )\\
= &\log\left( \|\zeta(u^{t+1}, v^{t+1}, U^t)\|_1\right) - \log\left(\|\zeta(u^{t+1}, v^t, U^t)\|_1 \right)+ \langle c, v^t - v^{t+1} \rangle\\
= &  \langle c, v^t - v^{t+1} \rangle\\
= & -\sum_j c_j  \log\left(\frac{c_j}{\kappa(\pi{(u^{t+1}, v^t, U^t)})_j}\right)\\
= & - \KCal( c || \kappa(\pi{(u^{t+1}, v^t, U^t)})) \\
\leq & - \frac{1}{2} \|\kappa(\pi{(u^{t+1}, v^t, U^t)}) - c\|_1^2,
\end{align*}}
where the last inequality is due to the Pinsker's inequality, and $\KCal(p||q) \ \mydefn \ \sum_{i = 1}^n p_i \log(\frac{p_i}{q_i})$ denotes the KL-divergence of $p$ and $q$. The proof is thus completed.
\end{proof}

Now we are ready to present the iteration complexity result for RBCD (Algorithm \ref{alg:RBCD}).
\begin{theorem} \label{thm:dualprwmain}
{
Choose parameters
\be\label{thm:dualprwmain-param}
\tau = \frac{1}{4  L_2 \|C\|_\infty /\eta + \rho L_1^2}, \quad \eta=\frac{\epsilon_2}{4\log(n)+2}, \quad \rho = \frac{2\|C\|_\infty}{\eta} + \frac{4\|C\|^2_\infty}{\eta^2}.
\ee
The Algorithm \ref{alg:RBCD} terminates (i.e., \eqref{RBCD-stopping} is satisfied) in
\be\label{thm:dualprwmain-T}
T = O\left(\log(n)\left(\frac{1}{\epsilon_2^3} + \frac{1}{\epsilon_1^2\epsilon_2} \right)\right)
\ee
iterations, where the $O(\cdot)$ hides constants related to $L_1$, $L_2$ and $\|C\|_\infty$ only. }
\end{theorem}

\begin{proof}
{By combining Lemmas \ref{lem:decU}, \ref{lem:decinu} and \ref{lem:decinv}, we have:
\begin{align}\label{thm:dualprwmain-eq-1}
& g(u^{t+1}, v^{t+1}, U^{t+1}) - g(u^{t}, v^{t}, U^t)\\
\le & - \left( \frac{1}{2} \|\kappa(\pi{(u^{t+1}, v^t, U^t)}) - c\|_1^2+ \frac{1}{8L_2\|C\|_\infty/\eta + 2\rho L_1^2 } \|\xi^{t+1}\|^2_F\right).\nonumber
\end{align}
Suppose Algorithm \ref{alg:RBCD} terminates at the $T$-th iteration. Summing \eqref{thm:dualprwmain-eq-1} over $t=0,\ldots,T-1$ yields
{
\begin{align}\label{eq:telescoping}
    & g(u^{T}, v^{T}, U^{T}) -  g(u^{0}, v^{0}, U^{0}) \\
\le & - \sum_{t= 0}^{T-1} \left( \frac{1}{2} \|\kappa(\pi{(u^{t+1}, v^t, U^t)}) - c\|_1^2+ \frac{1}{8L_2\|C\|_\infty/\eta + 2\rho L_1^2 } \|\xi^{t+1}\|^2_F\right)\nonumber \\
= &  - \sum_{t= 0}^{T-1} \left( \frac{1}{2} \|\kappa(\pi{(u^{t+1}, v^t, U^t)}) - c\|_1^2 + \frac{\eta^2\|\xi^{t+1}\|^2_F}{(8L_2\|C\|_\infty + 4 L_1^2 \|C\|_\infty)\eta + 8L_1^2\|C\|^2_\infty  } \right)\nonumber \\
\le &  - \sum_{t= 0}^{T-1} \min\left\{\frac{1}{2},\frac{1}{(8L_2\|C\|_\infty + 4 L_1^2 \|C\|_\infty)\eta + 8L_1^2\|C\|^2_\infty  }\right\}\cdot \left( \|\kappa(\pi{(u^{t+1}, v^t, U^t)}) - c\|_1^2 + \eta^2\|\xi^{t+1}\|^2_F\right) \nonumber \\
\le & -T \cdot \min\left\{\frac{1}{2},\frac{1}{(8L_2\|C\|_\infty + 4 L_1^2 \|C\|_\infty)\eta + 8L_1^2\|C\|^2_\infty }\right\}\cdot\min\left\{ \left(\frac{\epsilon_1}{4}\right)^2, \left(\frac{\epsilon_2}{8\|C\|_\infty}\right)^2\right\}, \nonumber
\end{align}
where the equality is obtained by plugging in the definition of $\rho$ in \eqref{thm:dualprwmain-param}, and the last inequality follows from the fact that \eqref{RBCD-stopping} does not hold for $t<T$.
By combining with \eqref{eq:gbound} and \eqref{thm:dualprwmain-param}, \eqref{eq:telescoping} immediately leads to
\begin{align}\label{thm:dualprwmain-eq-2}
T & \le (g( u^{0}, v^{0}, U^{0}) - g^*) \cdot\max\left\{2,(8L_2\|C\|_\infty + 4 L_1^2 \|C\|_\infty)\eta + 8L_1^2\|C\|^2_\infty \right\} \nonumber \\  & \cdot\max\left\{\left(\frac{4}{\epsilon_1}\right)^2,\left(\frac{8\|C\|_\infty}{\epsilon_2}\right)^2\right\}  \\
& \le \left(g( u^{0}, v^{0}, U^{0}) + \frac{\|C\|_\infty}{\eta}\right) \cdot\max\left\{2,(8L_2\|C\|_\infty + 4 L_1^2 \|C\|_\infty)\eta + 8L_1^2\|C\|^2_\infty \right\} \nonumber \\ & \cdot\max\left\{\left(\frac{4}{\epsilon_1}\right)^2,\left(\frac{8\|C\|_\infty}{\epsilon_2}\right)^2\right\} \nonumber \\
& = O\left(\log(n)\left(\frac{1}{\epsilon_2^3} + \frac{1}{\epsilon_1^2\epsilon_2} \right)\right). \nonumber
\end{align}
}
This completes the proof.
}
\end{proof}

{The next theorem gives the total number of arithmetic operations for Algorithm \ref{alg:RBCD}.}

\begin{theorem} \label{thm:dualprwmain-arithmetic}
{The Algorithm \ref{alg:RBCD} returns an $(\epsilon_1,\epsilon_2)$-stationary point of the PRW problem \eqref{PRW} in
\be\label{thm:dualprwmain-T}
O\left((n^2dk+dk^2 + k^3)\log(n)\left(\frac{1}{\epsilon_2^3} + \frac{1}{\epsilon_1^2\epsilon_2}\right)\right)
\ee
arithmetic operations, where the ${O}(\cdot)$ hides constants related to $L_1$, $L_2$ and $\|C\|_\infty$ only.}
\end{theorem}
\begin{proof}
In each iteration of Algorithm \ref{alg:RBCD}, we need to conduct the calculations in \eqref{RBCD-3-steps}. First, note that \eqref{RBCD-3-steps-1} and \eqref{RBCD-3-steps-2} can be done in $O(n)$ arithmetic operations \cite{Dvurechensky-icml-2018,Lin-icml-2019}. Second, the retraction step \eqref{RBCD-3-steps-3-3} requires $O(dk^2 + k^3)$ arithmetic operations if the one based on QR deomposition or polar decomposition is used \cite{chen2020proximal}. Third, note that we actually do not need to explicitly compute $V_\pi$ in \eqref{RBCD-3-steps-3-1}, and we only need to compute $V_\pi U$ in \eqref{RBCD-3-steps-3-2}, which can be done in $O(n^2dk)$ arithmetic operations \cite{lin2020projection}. Therefore, the per-iteration complexity of arithmetic operations of Algorithm \ref{alg:RBCD} is $O(n^2dk+dk^2 + k^3)$. Combining with Theorem \ref{thm:dualprwmain} yields the desired result.

\end{proof}

\begin{remark}
Note that the complexity of arithmetic operations of our RBCD is significantly better than the corresponding complexity of RGAS \cite{lin2020projection}, which is\footnote{This result was not in the published paper \cite{lin2020projection}. It is a refined result appeared later in an updated arxiv version \cite{lin2020projection-arxiv}.}
\[O(n^2d\|C\|_\infty^4\epsilon^{-4} + n^2\|C\|_\infty^8\epsilon^{-8} + n^2\|C\|_\infty^{12}\epsilon^{-12}).\]
When $\epsilon_1=\epsilon_2=\epsilon$, our complexity bound \eqref{thm:dualprwmain-T} reduces to
{
\[ O((n^2dk+dk^2 + k^3)\log(n) \epsilon^{-3}).\]}
Since $k=O(1)$, we conclude that our complexity bound is significantly better than the one in \cite{lin2020projection}.
\end{remark}

\section{Riemannian Adaptive Block Coordinate Descent Algorithm}\label{sec:RABCD}

In this section, we propose a variant of RBCD that incorporates an adaptive updating strategy for the Riemannian gradient step. In the Euclidean setting, this adatptive updating strategy employs different learning rate for each coordinate of the variable. Some well-known algorithms in this class include AdaGrad \cite{Duchi-AdaGrad-2011} and ADAM \cite{ADAM-2015}. This strategy was extended to the Riemannian setting by \cite{kasai2019riemannian}. Lin \etal \cite{lin2020projection} adopted this strategy and designed a variant of their RGAS algorithm for computing PRW, named Riemannian Adaptive Gradient Ascent with Sinkhorn (RAGAS), and they showed that numerically RAGAS is usually faster than RGAS. Motivated by these existing works, we propose a Riemannian adaptive block coordinate descent algorithm (RABCD) for computing PRW. In our numerical experiments presented in Section \ref{sec:num}, we also observe that RABCD is usually faster than RBCD.

We now briefly introduce the adaptive strategy proposed in \cite{kasai2019riemannian}. For Riemannian optimization problem $\min\{f(U)\mid U\in\St(d,k)\}$, the Riemannian adaptive gradient descent algorithm proposed in  \cite{kasai2019riemannian} uses separate adaptive diagonal weight matrices for row and column subspaces of the Riemannian gradient $\text{grad} f(U)$. The diagonal weighted matrices, denoted as $\text{Diag}(p) \in \br^{d\times d}, \text{Diag}(q) \in \br^{k \times k}$, are computed by all previous Riemannian gradients in an exponentially weighted form:
\be\label{eq:weightedmat}
\bad
\text{Diag}(p^{t+1}) := \beta  \text{Diag}(p^{t}) + (1-\beta) \text{diag}(G^{t+1}(G^{t+1})^\top) / k,\\
\text{Diag}(q^{t+1}) := \beta  \text{Diag}(q^{t}) + (1-\beta) \text{diag}((G^{t+1})^\top G^{t+1}) /d,
\ead
\ee
where $\beta \in (0,1)$ is a parameter and $G^{t+1}(G^{t+1})^\top/ k$, $(G^{t+1})^\top G^{t+1}/d$ are the row and column covariance matrices of the Riemannian gradient, respectively. The following non-decreasing sequence of adaptive weights is computed for the purpose of convergence guarantee in \cite{kasai2019riemannian}:
\be\label{eq:weightedmatlowerbound}
\hat{p}^{t+1} = \max\{\hat{p}^{t}, p^{t+1}\}, \quad \hat{q}^{t+1} = \max\{\hat{q}^{t}, q^{t+1}\},
\ee
with the initial values $\hat{p}_0 = \alpha\|C\|_\infty^2 \textbf{1}_d, \hat{q}_0 = \alpha\|C\|_\infty^2 \textbf{1}_k$, where $\textbf{1}_d$ denotes the all-one vector with dimension $d$. Finally, the adaptive Riemannian gradient can be computed as
\be\label{eq:adaprgd}
\bad
\xi^{t+1}  = \proj_{\T_{U^t}\St(d,k)}(\text{Diag} (\hat{p}^{t+1})^{-1/4} G^{t+1} \text{Diag} (\hat{q}^{t+1})^{-1/4}).
\ead
\ee
Our RABCD algorithm employs exactly the same strategy for the Riemannian gradient step, i.e., update of $U$. The $t$-th iteration of our RABCD algorithm can be described as follows:
\begin{subequations}\label{RABCD-3-steps}
{
\begin{align}
u^{t+1} & := u^t + \log(r./\varphi(\zeta(u^t,v^t,U^t))) \label{RABCD-3-steps-1-sol} \\
v^{t+1} & := v^t + \log(c./\kappa(\zeta(u^{t+1},v^t,U^t))) \label{RABCD-3-steps-2-sol} \\
V_{\pi(u^{t+1},v^{t+1},U^t)} & := \sum_{ij} [\pi(u^{t+1}, v^{t+1}, U^t)]_{ij} (x_i - y_j) (x_i - y_j)^\top \label{RABCD-3-steps-3-1} \\
G^{t+1} & :=  \proj_{T_{U^t}\M}(-2V_{\pi(u^{t+1},v^{t+1},U^t)}U^t) \label{RABCD-3-steps-3-2}\\
p^{t+1}  & := \beta p^t + (1-\beta) \text{diag} (G^{t+1}(G^{t+1})^\top)/k, \quad \hat{p}^{t+1} = \max \{\hat{p}^{t}, p^{t+1} \}\label{RABCD-3-steps-3-3}\\
q^{t+1}  & := \beta q^t + (1-\beta) \text{diag} ((G^{t+1})^\top G^{t+1})/d, \quad \hat{q}^{t+1} = \max \{\hat{q}^{t}, q^{t+1} \} \label{RABCD-3-steps-3-4} \\
\xi^{t+1}  & :=\eta^{-1} \cdot \proj_{\T_{U^t}\M}(\text{Diag} (\hat{p}^{t+1})^{-1/4} G^{t+1} \text{Diag} (\hat{q}^{t+1})^{-1/4})\label{RABCD-3-steps-3-5}\\
U^{t+1} & := \retr_{U^t}( - \tau\xi^{t+1}). \label{RABCD-3-steps-3-6}
\end{align}}
\end{subequations}
We terminate the RABCD algorithm when the following inequalities are satisfied simultaneously:
{
\be\label{RABCD-stopping}
\|G_{t+1}\|_F \le \frac{\epsilon_1}{4}, \quad \|c - \kappa(\zeta(u^{t+1}, v^{t}, U^t))\|_1 \le \frac{\epsilon_2}{8\|C\|_\infty},
\ee}
where $\epsilon_1, \epsilon_2$ are pre-given accuracy tolerances.
Details of RABCD algorithm are given in Algorithm \ref{alg:RABCD}.
\begin{algorithm}[h]
\caption{ Riemannian Adaptive Block Coordinate Descent Algorithm (RABCD) }
\label{alg:RABCD}
\begin{algorithmic}[1]
\STATE \textbf{Input:}  $\{(x_i,r_i)\}_{i\in [n]}$ and $\{(y_j,c_j)\}_{j\in [n]}$, $U^0\in\M$, $u^0, v^0\in\br^n$, $p^0 = \textbf{0}_d$, $q^0 = \textbf{0}_k$, $ \hat{p}^0 = \alpha\|C\|_\infty^2 \textbf{1}_d$, $\hat{q}^0 = \alpha\|C\|_\infty^2 \textbf{1}_k$. $\alpha, \beta \in (0, 1)$, and accuracy tolerance $\epsilon_1\geq\epsilon_2>0$. Set parameters as in \eqref{thm:dualadaptiveprwmain-param}.
\FOR{$t = 0, 1, 2, \ldots$}
\STATE Conduct the computation in \eqref{RABCD-3-steps}.
\IF{\eqref{RABCD-stopping} is satisfied}
\STATE break;
\ENDIF
\ENDFOR
\STATE \textbf{Output:} $\hat{u} = u^{t+1}$, $\hat{v} = v^t$,  $\hat{U} = U^t$, and $\hat{\pi}= Round(\pi(\hat{u}, \hat{v}, \hat{U}), r, c)$.
\end{algorithmic}
\end{algorithm}

We now present the convergence analysis of the RABCD algorithm. Since RABCD has the same Sinkhorn steps as the RBCD, Lemmas \ref{lem:decinu} and \ref{lem:decinv} still apply here. Moreover, since we still have the same objective function, it is easy to verify that Lemma \ref{lem:lipschitz} also holds here. It only remains to show that the objective function $g$ has sufficient reduction when $U$ is updated in \eqref{RABCD-3-steps}.
\begin{theorem} \label{thm:dualadaptiveprwmain}
Choose parameters
{
\be\label{thm:dualadaptiveprwmain-param}
\tau = \frac{\alpha \|C\|_\infty}{8L_2 \|C\|_\infty /\eta  + 2\rho L_1^2}, \quad \eta=\frac{\epsilon_2}{4\log(n)+2}, \quad \rho = \frac{2\|C\|_\infty}{\eta} + \frac{4\|C\|^2_\infty}{\eta^2}.
\ee
Algorithm \ref{alg:RABCD} terminates (i.e., \eqref{RABCD-stopping} is satisfied) in
\be\label{thm:dualadaptiveprwmain-T}
T = O\left(\frac{\log(n) }{\alpha} \left(\frac{1}{\epsilon_2^3} + \frac{1}{ \epsilon_1^2\epsilon_2}\right)\right)
\ee}
iterations, where the $O(\cdot)$ hides constants related to $L_1$, $L_2$ and $\|C\|_\infty$ only. That is, the RABCD algorithm has the same order of iteration complexity as RBCD.
\end{theorem}

\begin{proof}
As discussed before, we only need to show that the objective function g has a sufficient reduction when the variable $U$ is updated in Algorithm \ref{alg:RABCD}.  Combining \eqref{eq:lipineq} and the second last equation of \eqref{eq:lipinnerterm}, we have shown that
\be\label{eq:adapdecg}\bad
&g(u^{t+1}, v^{t+1}, U^{t+1}) - g(u^{t+1}, v^{t+1}, U^t)\\
\le & - \tau \langle \nabla_U g(u^{t+1}, v^{t+1}, U^t),  \xi^{t+1} \rangle + \left( \frac{2}{\eta}  L_2  \|C\|_\infty + \frac{\rho}{2} L_1^2\right) \tau^2\|\xi^{t+1}\|_F^2.
\ead\ee
Here in the RABCD algorithm, we have
\[\xi^{t+1}  = \eta^{-1} \cdot \proj_{T_{U^t}\M}(\text{Diag} (\hat{p}^{t+1})^{-1/4} G^{t+1} \text{Diag} (\hat{q}^{t+1})^{-1/4}), \quad G^{t+1} =  \eta \cdot \text{grad}_U g(u^{t+1}, v^{t+1}, U^t).
\]
Using \eqref{sinkhorn-update-satisfy-L1}, we can bound $\|G^{t+1}\|$ as
\[
\|G^{t+1}\|_F = \| \proj_{T_{U^t}\St}( - 2V_{\pi(u^{t+1},v^{t+1},U^t)}U^t)\|_F \le 2 \|V_{\pi(u^{t+1},v^{t+1},U^t)}U^t\|_F \le 2\|C\|_\infty,
\]
which leads to
\[
\textbf{0}_d \leq \frac{\diag(G^{t+1}(G^{t+1})^\top)}{k} \leq 4\|C\|_\infty^2 \textbf{1}_d,   \quad \textbf{0}_k \leq \frac{\diag((G^{t+1})^\top G^{t+1})}{d} \leq 4\|C\|_\infty^2 \textbf{1}_k.
\]
Here $\leq$ denotes an element-wise comparison between two vectors. By the definitions of $p, q$, we have
\[
\textbf{0}_d \leq p^t \leq 4\|C\|_\infty^2 \textbf{1}_d,   \quad \textbf{0}_k \leq q^t \leq 4\|C\|_\infty^2 \textbf{1}_k.
\]
So we have the following bound for $\hat{p}, \hat{q}$:
\be\label{bound-hatp-hatq}
\alpha \|C\|_\infty^2 \textbf{1}_d \leq \hat{p}^t \leq 4\|C\|_\infty^2 \textbf{1}_d,   \quad  \alpha \|C\|_\infty^2 \textbf{1}_k \leq \hat{q}^t \leq 4\|C\|_\infty^2 \textbf{1}_k.
\ee
{
Now we are ready to bound \eqref{eq:adapdecg}. For the first term on the right hand side of \eqref{eq:adapdecg}, we have
\be\label{eq:adapinner}\bad
- \langle \nabla_U g(u^{t+1}, v^{t+1}, U^t),  \xi^{t+1} \rangle &= - \langle \text{grad}_U g(u^{t+1}, v^{t+1}, U^t),  \xi^{t+1} \rangle\\
& = - \eta^{-1}  \langle \text{grad}_U g(u^{t+1}, v^{t+1}, U^t),  \text{Diag} (\hat{p}^{t+1})^{-1/4} G^{t+1} \text{Diag} (\hat{q}^{t+1})^{-1/4} \rangle\\
& \le - \frac{1}{2\|C\|_\infty} \| \text{grad}_U g(u^{t+1}, v^{t+1}, U^t)\|_F^2,
\ead\ee
where the inequality is due to \eqref{bound-hatp-hatq}.
For the second term on the right hand side of \eqref{eq:adapdecg}, we have
\be\label{eq:adapnormxi}\bad
\|\xi^{t+1}\|_F^2  \le \frac{1}{\alpha\|C\|_\infty^2} \| \text{grad}_U g(u^{t+1}, v^{t+1}, U^t)\|_F^2.
\ead\ee
Combining \eqref{eq:adapinner}, \eqref{eq:adapnormxi} and \eqref{eq:adapdecg} yields
\be\label{eq:adapdecgfinal}\bad
&g(u^{t+1}, v^{t+1}, U^{t+1}) - g(u^{t+1}, v^{t+1}, U^t)\\
\le & - \tau\left(\frac{1}{2\|C\|_\infty} - \left(\frac{2L_2}{\alpha\eta\|C\|_\infty} + \frac{\rho L_1^2  }{2\alpha \|C\|_\infty^2} \right)  \tau\right) \| \text{grad}_U g(u^{t+1}, v^{t+1}, U^t)\|_F^2.
\ead\ee
Choosing $\tau$ as in \eqref{thm:dualadaptiveprwmain-param}, we have
\be\label{eq:adapdecgfinal}\bad
&g(u^{t+1}, v^{t+1}, U^{t+1}) - g(u^{t+1}, v^{t+1}, U^t)\le - \frac{\alpha}{32 \|C\|_\infty L_2 /\eta  + 8 \rho L_1^2 } \| \text{grad}_U g(u^{t+1}, v^{t+1}, U^t)\|_F^2.
\ead\ee
{
Suppose Algorithm \ref{alg:RABCD} terminates at the $T$-th iteration. By combining \eqref{eq:adapdecgfinal} with Lemmas \ref{lem:decinu} and \ref{lem:decinv}, and using the parameters settings in \eqref{thm:dualadaptiveprwmain-param}, we have
\be \bad\label{eq:telescopingadap}
&g(u^{T}, v^{T}, U^{T}) -  g(u^{0}, v^{0}, U^{0}) \\
\le & - \sum_{t= 0}^{T-1} \left(  \frac{1}{2} \|\kappa(\pi{(u^{t+1}, v^t, U^t)}) - c\|_1^2+ \frac{\alpha\| \text{grad}_U g(u^{t+1}, v^{t+1}, U^t)\|_F^2}{32 \|C\|_\infty L_2 /\eta  + 8 \rho  L_1^2 } \right)\\
= & - \sum_{t= 0}^{T-1} \left(  \frac{1}{2} \|\kappa(\pi{(u^{t+1}, v^t, U^t)}) - c\|_1^2+ \frac{\alpha \eta^2\| \text{grad}_U g(u^{t+1}, v^{t+1}, U^t)\|_F^2}{(32\|C\|_\infty L_2  + 16\|C\|_\infty L_1^2)\eta + 32 \|C\|_\infty^2 L_1^2 } \right)\\
\le & - \sum_{t= 0}^{T-1} \min\left\{ \frac{1}{2}, \frac{\alpha}{(32\|C\|_\infty L_2  + 16\|C\|_\infty L_1^2)\eta + 32 \|C\|_\infty^2 L_1^2 } \right\}\cdot\\
& \left( \|\kappa(\pi{(u^{t+1}, v^t, U^t)}) - c\|_1^2+ \eta^2\| \text{grad}_U g(u^{t+1}, v^{t+1}, U^t)\|_F^2\right)\\
\le & - T \cdot \min\left\{ \frac{1}{2}, \frac{\alpha}{(32\|C\|_\infty L_2  + 16\|C\|_\infty L_1^2)\eta + 32 \|C\|_\infty^2 L_1^2  } \right\} \cdot \min \left\{\left(\frac{\epsilon_1 }{4} \right)^2  , \left(\frac{\epsilon_2 }{8\|C\|_\infty}\right)^2 \right\},
\ead\ee
where the last inequality follows from the fact that \eqref{RABCD-stopping} does not hold for $t<T$. Finally, \eqref{eq:telescopingadap} leads to
\begin{align}\label{thm:dualprwmain-eq-2}
T & \le (g( u_{0}, v_{0}, U_{0}) - g^*) \cdot \max\left\{2 , \frac{ (32\|C\|_\infty L_2  + 16\|C\|_\infty L_1^2)\eta + 32\|C\|_\infty^2 L_1^2  }{\alpha}\right\} \\ & \cdot \left\{ \left(\frac{4 }{\epsilon_1}\right)^2,   \left(\frac{8\|C\|_\infty }{\epsilon_2 }\right)^2\right\}\nonumber \\
& \le  (g( u_{0}, v_{0}, U_{0})  + \frac{\|C\|_\infty}{\eta}) \cdot \max\left\{2 ,  \frac{ (32\|C\|_\infty L_2  + 16\|C\|_\infty L_1^2)\eta + 32 \|C\|_\infty^2 L_1^2 }{\alpha}\right\} \\ & \cdot \left\{ \left(\frac{4 }{\epsilon_1}\right)^2,   \left(\frac{8\|C\|_\infty }{\epsilon_2 }\right)^2\right\}\nonumber \\
& = O\left(\frac{\log(n) }{\alpha} \left(\frac{1}{\epsilon_2^3} + \frac{1}{ \epsilon_1^2\epsilon_2}\right)\right),\nonumber
\end{align}
where the last step follows from the choice of $\eta$ in \eqref{thm:dualadaptiveprwmain-param}.
}}
\end{proof}

\begin{remark}
Note that the per-iteration complexity of RABCD is the same as RBCD. Therefore, by a similar argument as in Theorem \ref{thm:dualprwmain-arithmetic}, we know that the total number of arithmetic operations of RABCD is
{
\[O\left((n^2dk+dk^2 + k^3)\frac{\log(n) }{\alpha} \left(\frac{1}{\epsilon_2^3} + \frac{1}{ \epsilon_1^2\epsilon_2}\right)\right).\] }
We skip details for succinctness.
\end{remark}

\section{Numerical Experiments}\label{sec:num}

In this section, we evaluate the performance of our proposed RBCD algorithm on calculating the PRW distance for both synthetic and real datasets. We mainly focus on the comparison of the computational time between the RBCD algorithm and the RGAS algorithm \cite{lin2020projection}, which is currently the state of the art algorithm for solving the PRW problem. All experiments in this section are implemented in Python 3.7 on a linux server with a 32-core Intel Xeon CPU (E5-2667, v4, 3.20GHz per core).

\subsection{Synthetic Datasets}

We first focus on two synthetic examples, which are adopted in \cite{paty2019subspace,lin2020projection} and their ground truth Wasserstein distance can be computed analytically.

\paragraph{Fragmented Hypercube:} We consider a uniform distribution over a hypercube $\mu \ = \ \UCal([-1, 1]^d)$ and a pushforward $\nu = T_\#\mu$ defined under the map $T(x) = x + 2 \sign(x) \odot (\sum_{k = 1}^{k^*}e_k)$, where $\sign(\cdot)$ is taken element-wise, $k^* \in [d]$ and $e_i,\  i \in [d]$ is the canonical basis of $\br^d$. The pushforward $T$ splits the hypercube into {$2^{k^*}$} different hyper rectangles. Since $T$ can be viewed as the subgradient of a convex function, \cite{brenier1991polar} has shown that $T$ is an optimal transport map between $\mu$ and $\nu = T_\#\mu$ with $\WCal(\mu, \nu)^2 = 4k^*$. In this case, the displacement vector $T(x) - x$ lies in the $k^*$-dimensional subspace spanned by $\{e_j\}_{j \in [k^*]}$ and we should have $\PCal_k^2 = 4k^*$ for any $k \ge k^*$. Moreover, in this case we have {$U^*\in\St(d,k^*)$ with $U^*(1:k^*,1:k^*) = I_{k^*}$}.  For all experiments in this subsection, we set the parameters as $\eta = 0.2$, $\epsilon_{RGAS} = \epsilon_1 = \epsilon_2 = 0.1, \tau_{RGAS} = \tau_{RBCD}/\eta $
and $\tau_{RBCD}= 0.005$. Figure \ref{fig:prw_on_k} shows the computation of $\PCal_k^2(\hat{\mu},\hat{\nu})$ on different $k$ with $k^* \in \{2, 4, 7, 10\}$. After setting $n = 100, d = 30$ and generating the Fragmented Hypercube data with different $k^*$, we run both the RBCD and the RGAS \cite{lin2020projection} algorithms for calculating the PRW distance. We see that the PRW value $\PCal_k^2(\hat{\mu},\hat{\nu})$ grows more slowly after $k = k^*$ for both algorithms, which is reasonable since the last $d - k^*$ dimensions only represent noise. Furthermore, $\PCal_k^2(\hat{\mu},\hat{\nu}) \approx 4k^*$ holds when $k = k^*.$ Finally, we see that the solutions of both the RBCD and the RGAS algorithms achieve almost the same quality.

\begin{figure}[ht]
\centering\hspace*{-1em}
\includegraphics[clip, trim=0 0 0 0, width=.96\textwidth]{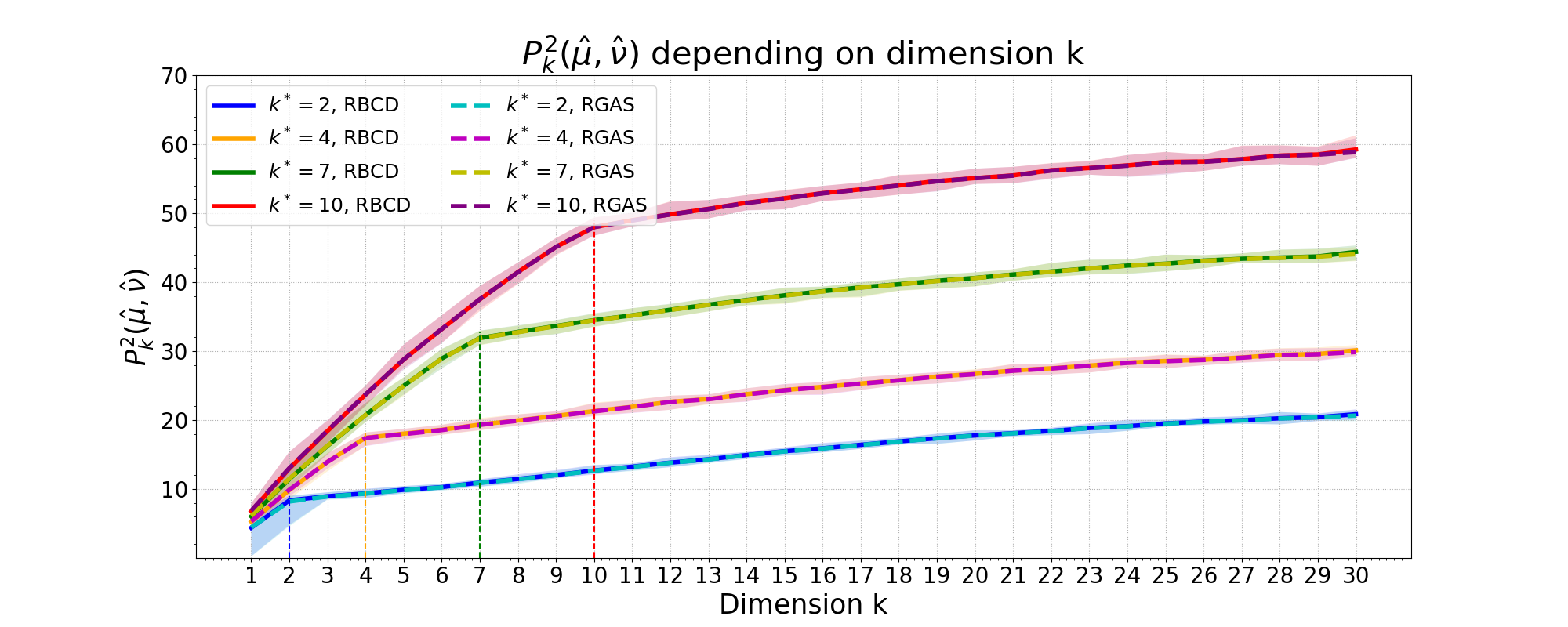}
\vspace*{-.5em}\caption{Computation of PRW value $\PCal_k^2(\hat{\mu},\hat{\nu})$ depending on the dimension $k \in [d]$ and $k^* \in \{2, 4, 7, 10\}$, where $\hat{\mu}$ and $\hat{\nu}$ stand for the empirical measures of $\mu$ and $\nu$ with $n = 100, d = 30$. The solid and dash curves are the computation of $\PCal_k^2(\hat{\mu},\hat{\nu})$ with the RBCD and RGAS algorithms, respectively. Each curve is the mean over 100 samples with shaded area covering the min and max values.}
\label{fig:prw_on_k}\vspace*{-1em}
\end{figure}

We present in Figure \ref{fig:prw_on_n} the mean estimation error (MEE) for the sampled PRW distance $\PCal_k^2(\hat{\mu},\hat{\nu})$ and the sampled Wasserstein distance $\WCal^2(\hat{\mu},\hat{\nu})$ for different choices of $n\in\{25, 50, 100, 250, 500, 1000\}$. Theoretically, the MEE of both the PRW distance and the Wasserstein distance decreases as $n$ increases. However, \cite{niles2019estimation} showed that for spiked transport model the convergence rates of the estimation error are different. Specifically, when $d > 4$, we have for the sampled Wasserstein distance,
$$\mathbb{E}\lvert \WCal(\mu, \nu) - \WCal(\hat{\mu}, \hat{\nu})\rvert = \bigO(n^{-1/d}),$$
and for the PRW distance,
$$\mathbb{E}\lvert \WCal(\mu, \nu) - \PCal_k(\hat{\mu},\hat{\nu})\rvert = \bigO(n^{-1/k}),$$
which significantly alleviate the curse of dimensionality because $k\ll d$. We set $k^* = 2, d = 20$ and generate $(\hat{\mu}, \hat{\nu})$ from $(\mu, \nu)$ with $n$ points. We calculate the estimation error in each run as $MEE =\mathbb{E} \lvert \PCal_k^2(\hat{\mu},\hat{\nu}) - 4k^* \rvert$ for the PRW distance and $MEE = \mathbb{E}\lvert \WCal^2(\hat{\mu},\hat{\nu}) - 4k^* \rvert$ for the Wasserstein distance. For the PRW distance, we further show the mean subspace estimation error {$\|\hat{\Omega} - \Omega^*\|_F $} in Figure \ref{fig:prw_on_n}. The subspace projection can be calculated as $\hat{\Omega} = \hat{U}\hat{U}^\top$, where $\hat{U}$ is the output of the algorithm. From Figure \ref{fig:prw_on_n} we see that as $n$ increases, both the MEE and the mean subspace estimation error decrease for both RBCD and RGAS algorithms. Moreover, we see that Wasserstein distance estimation behaves much worse than the PRW distance when the same number of samples are used.

\begin{figure}[ht]
\centering\hspace*{-1em}
\includegraphics[width=.45\textwidth]{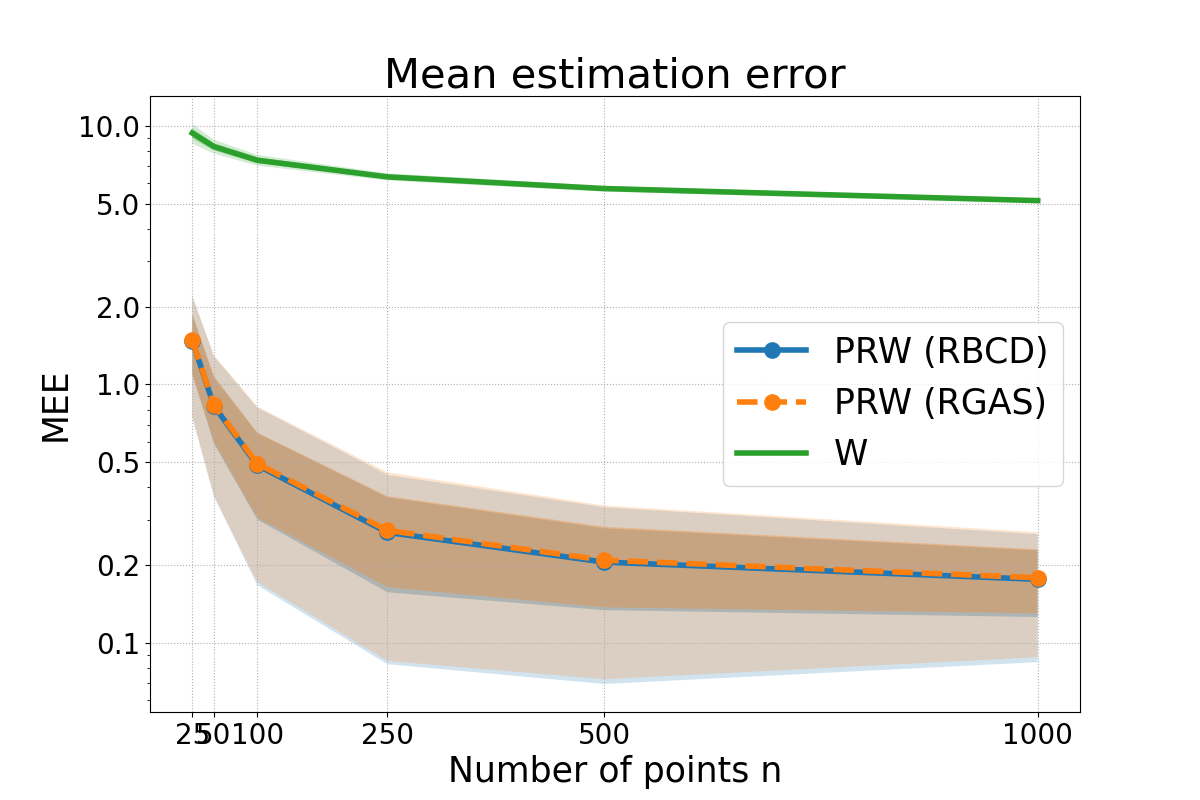}
\includegraphics[width=.45\textwidth]{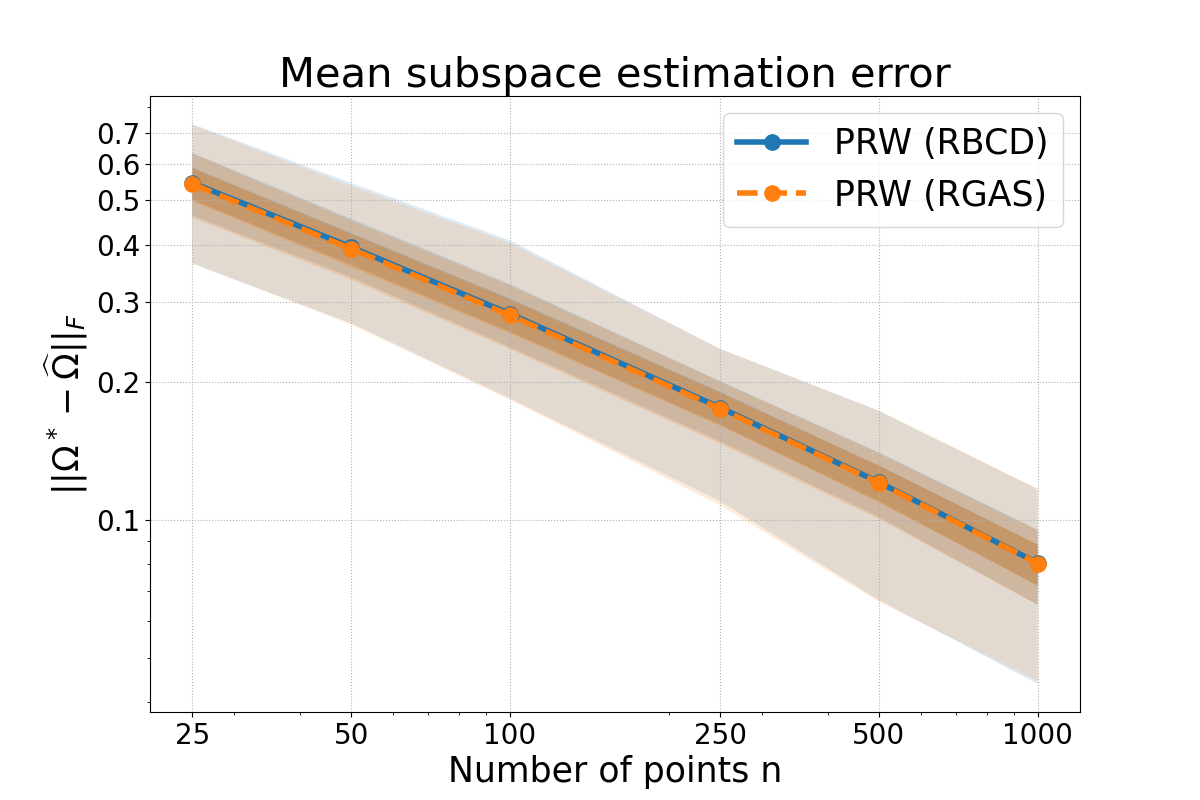}
\vspace*{-.5em}\caption{\textbf{Left:} The mean estimation error (MEE); \textbf{Right:} The mean subspace estimation error against the number of samples $n \in \{25, 50, 100, 250, 500, 1000\}.$ We set $k^* = 2, d = 30$ and calculate the mean estimation error as $MEE = \lvert \PCal_k^2(\hat{\mu},\hat{\nu}) - 4k^* \rvert$ for the PRW distance and $MEE = \lvert \WCal_k^2(\hat{\mu},\hat{\nu}) - 4k^* \rvert$ for the Wasserstein distance. The subspace projection is calculated as $\hat{\Omega} = \hat{U}\hat{U}^\top$ in each run. The shaded areas represent the 10\%-90\% and 25\%-75\% quantiles over 500 samples.}
\label{fig:prw_on_n}\vspace*{-1em}
\end{figure}

We also plot the optimal transport plans between $(\hat{\mu}, \hat{\nu})$ generated by the Wasserstein distance and the PRW distance calculated by the RGAS and RBCD algorithms. The results are shown in Figure \ref{fig:transport_plan}, where we considered the case when $k^* = 2, d = 30$ and $n \in \{100, 250\}$. From Figure \ref{fig:transport_plan} we see that in both cases, our RBCD algorithm can generate {almost the same transport plans as the RGAS algorithm, which are also similar to the transportation plan generated by the Wasserstein distance.}

\begin{figure}[ht]
\centering
\includegraphics[width=0.28\textwidth]{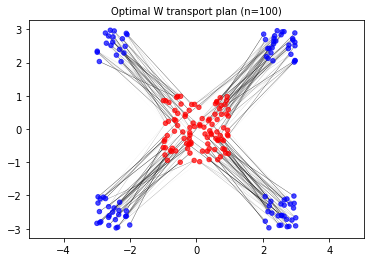}
\includegraphics[width=0.28\textwidth]{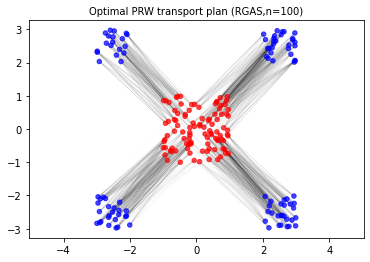}
\includegraphics[width=0.28\textwidth]{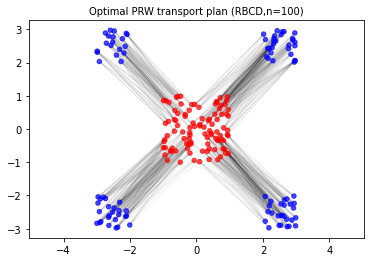}
\includegraphics[width=0.28\textwidth]{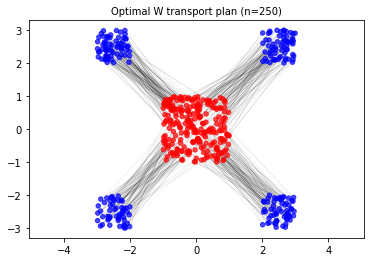}
\includegraphics[width=0.28\textwidth]{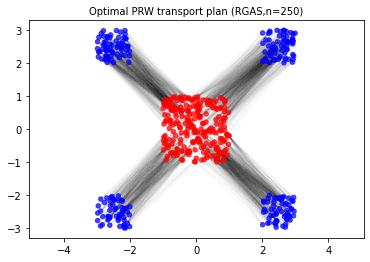}
\includegraphics[width=0.28\textwidth]{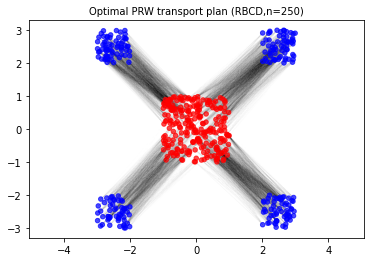}
\vspace*{-.5em}\caption{Fragmented hypercube with $(n, d) = (100, 30)$ (above) and $(n, d) = (250, 30)$ (bottom). Optimal transport plan obtained by the Wasserstein distance (left), the PRW distance calculated by the RGAS algorithm (middle) and the PRW distance calculated by the RBCD algorithm (right). Geodesics in the PRW space are robust to statistical noise.}
\label{fig:transport_plan}
\end{figure}

\paragraph{Gaussian Distribution:} We further conduct experiments on the Gaussian distribution. Specifically, we consider $\mu \in \NCal(0, \Sigma_1)$ and $\nu \in \NCal(0, \Sigma_2)$ with $\Sigma_1, \Sigma_2 \in \br^{d \times d}$ being positive definite with rank $k^*$, which leads to the support of the distributions $\mu$ and $\nu$ being $k^*$-dimensional subspaces of $\br^d$. Therefore, the union of $\mu$ and $\nu$ must lie in a $2k^*$-dimensional subspace, which yields $\PCal_k^2(\mu, \nu) = \WCal^2(\mu, \nu)$ for any $k \ge 2k^*.$ Utilizing the synthetic data generated by the Gaussian distribution, we test the robustness of the PRW distance calculated by the RBCD algorithm.

We first show the mean values of $\PCal_k^2(\hat{\mu}, \hat{\nu}) / \WCal^2(\hat{\mu}, \hat{\nu})$, where $(\hat{\mu}, \hat{\nu})$ are obtained by drawing $n = 100$ points from $ \NCal(0, \Sigma_1)$ and $ \NCal(0, \Sigma_2)$, against different $k$.  We set $d = 20$ and sample 100 independent couples of covariance matrices $(\Sigma_1, \Sigma_2)$ according to a Wishart distribution with $k^* = 5$ degrees of freedom. We then add white noise $\NCal(0,I_d)$ to each data point.  We set the parameters as $\eta = 1$, $\epsilon_{RGAS} = \epsilon_1 = \epsilon_2 = 0.1, \tau_{RGAS} = \tau_{RBCD}/\eta$ and $\tau_{RBCD}= 0.005$. Figure \ref{fig:noise_k} shows the curves for both noise-free and noisy data obtained by the RGAS algorithm and the proposed RBCD algorithm. We can see that when there is no noise, the RBCD algorithm can recover the ground truth Wasserstein distance when $k \ge 2k^*$. With moderate noise, the PRW distance calculated by the RBCD algorithm can still approximately recover the Wasserstein distance, which is consistent with both the SRW distance and the PRW distance calculated by the RGAS algorithm.

\begin{figure}[ht]
\centering\hspace*{-1em}
\includegraphics[width=.45\textwidth]{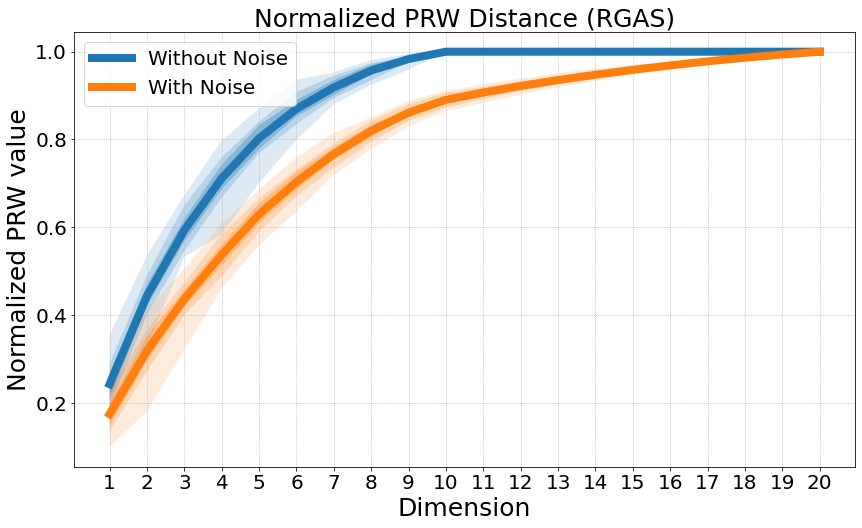}
\includegraphics[width=.45\textwidth]{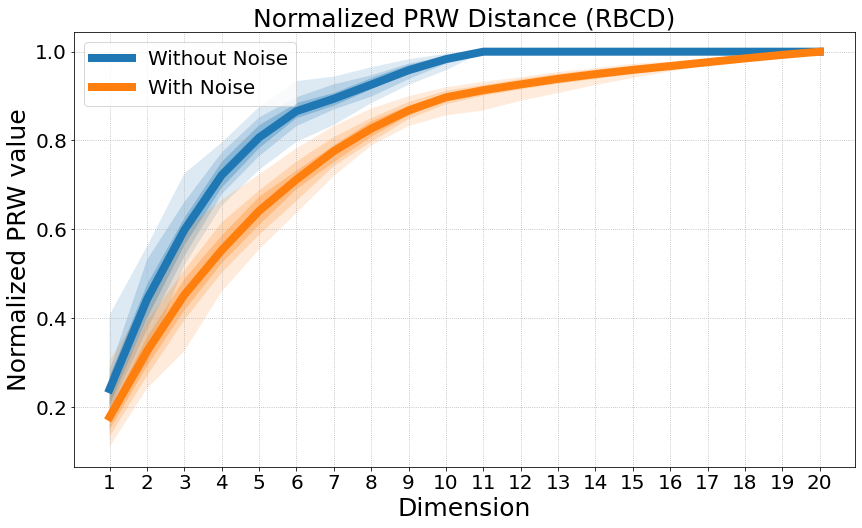}
\vspace*{-.5em}\caption{Mean normalized PRW distance calculated by the RGAS algorithm (left) and the RBCD algorithm (right) as a function of dimension $k$. We set $k^* = 5, d = 20, n = 100$. The shaded area shows the 10\%-90\% and 25\%-75\% quantiles over the 100 samples.}
\label{fig:noise_k}
\end{figure}

We further conduct experiments on testing the robustness of the PRW distance against the noise level. Specifically, we set $k^* = 5, d = 20, n = 100$ and sample 100 the Gaussian distribution with each couple of covariance matrices $(\Sigma_1, \Sigma_2)$ generated according to a Wishart distribution. We then gradually add Gaussian noise $\sigma \NCal(0, I),$ where $\sigma$ is the noise level and is chosen from $\{0.01, 0.1, 1, 2,4,7,10 \}$. In this experiment, we set the regularization coefficients as $\eta = 2$ when $\sigma \le 4$, and $\eta = 10$ otherwise. We set other parameters as $\epsilon_{RGAS} = \epsilon_1 = \epsilon_2 = 0.1, \tau_{RGAS} = \tau_{RBCD}/\eta $, and $\tau_{RBCD}= 0.01$ when $\sigma \le 4$, and $\tau_{RBCD} = 0.002$ otherwise. Figure \ref{fig:noise_level} presents the mean relative error for the Wasserstein distance and the PRW distance calculated by the RGAS and the RBCD algorithm with varying noise level. {The relative error is calculated by
$$\frac{\PCal_k^2(\hat{\mu}_\sigma, \hat{\nu}_\sigma)  - \PCal_k^2(\hat{\mu}_0, \hat{\nu}_0) }{\PCal_k^2(\hat{\mu}_0, \hat{\nu}_0) }$$ for the PRW distance and
$$\frac{\WCal_k^2(\hat{\mu}_\sigma, \hat{\nu}_\sigma)  - \WCal_k^2(\hat{\mu}_0, \hat{\nu}_0) }{\WCal_k^2(\hat{\mu}_0, \hat{\nu}_0) }$$
for the Wasserstein distance.} We see that all three distances behave similarly when the noise level is small. When the noise level $\sigma \ge 1$, the PRW distance calculated by both the RGAS and the RBCD algorithm outperform the standard Wasserstein distance, which further shows the robustness of the PRW distance against noise.

\begin{figure}[ht]
\centering\hspace*{-1em}
\includegraphics[clip, trim=0 0 0 0, width=.5\textwidth]{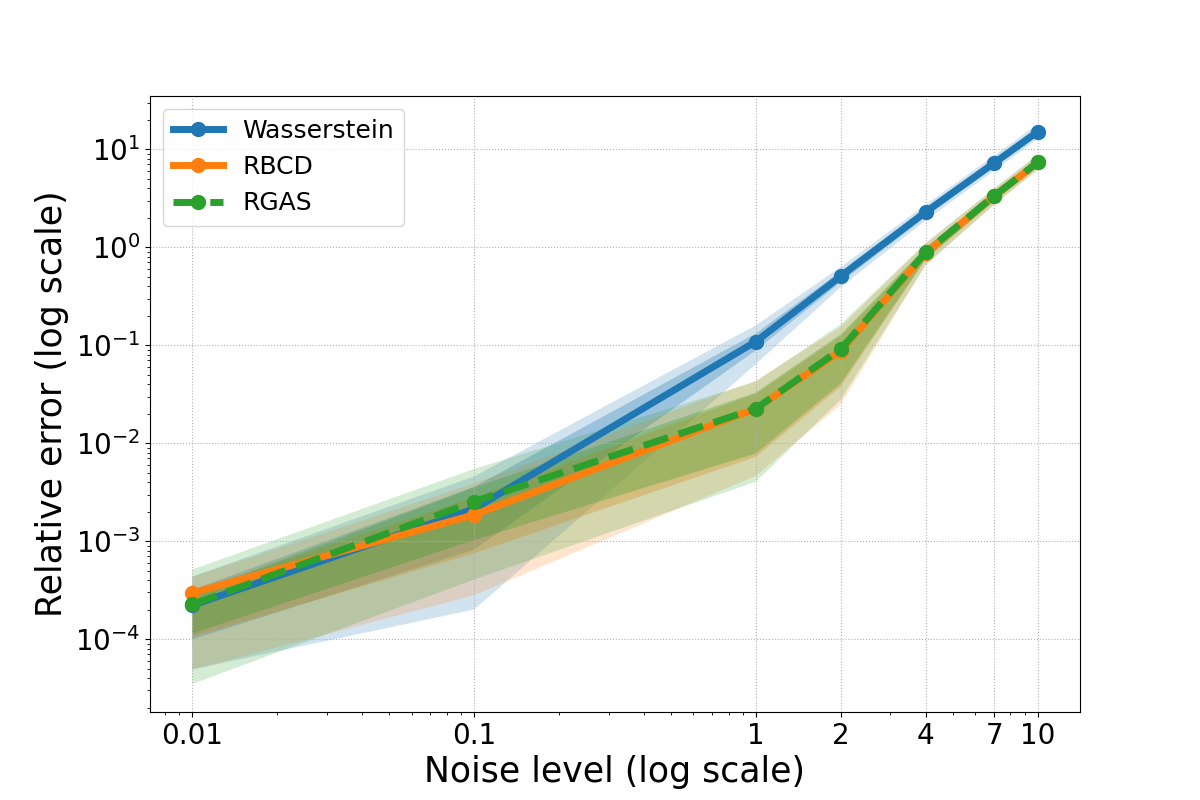}
\vspace*{-.5em}\caption{The mean relative error against the noise level $\sigma \in \{0.01, 0.1, 1, 2,4,7,10 \}.$ We set $k^* = 5, d = 20, n = 100$. The shaded area shows the 10\%-90\% and 25\%-75\% quantiles over the 100 samples.}
\label{fig:noise_level}\vspace*{-1em}
\end{figure}

\paragraph{Computational Time Comparison.} We compare the computational time of five different algorithms: the Frank-Wolfe (FW) algorithm for computing the SRW distance \cite{paty2019subspace}, the RGAS and the RAGAS algorithms proposed in \cite{lin2020projection}, and our RBCD and RABCD algorithms on computing the PRW distance for the two synthetic datasets mentioned above. {The RGAS and the RAGAS algorithms are terminated when the stopping criteria \ref{RBCD-stopping} and \ref{RABCD-stopping} are satisfied, respectively. The RGAS and the RAGAS algorithms are terminated when $\|\text{grad} p(U^t)\|_F \le \epsilon$. The FW algorithm is terminated when $\|\Omega^t - \Omega^{t-1}\|_F < (\epsilon\cdot\tau_{RBCD})^2$.}

We first consider the Fragmented Hypercube example. We fix $k^* = k = 2$, and generate the Fragmented Hypercube with varying $n, d$. We further set the thresholds $\epsilon_{RGAS} =\epsilon_{RAGAS} = \epsilon_1 = \epsilon = 0.1$ and $\epsilon_2 = \epsilon_1^2$. We set $\eta = 0.2$ when $d < 250$ and $\eta = 0.5$ otherwise. For fair comparison, we set the step size $\tau_{RGAS} =\tau_{RAGAS} = \tau_{RBCD}/\eta = \tau_{RABCD}/\eta $ and $\tau_{RBCD}= 0.001$. Tables \ref{tab:hctimefixn} - \ref{tab:hctimenequaltend} show the computational time comparison for different algorithms with different $(n, d)$ pairs. All the reported CPU times are in seconds. We first test the performance of our proposed algorithms when $n$ or $d$ is fixed. Specifically, in Table \ref{tab:hctimefixn}, we fix $n = 100$ and show the running time for different $d \in \{20, 50, 100, 250, 500\}.$ In Table \ref{tab:hctimefixd}, we fix $d = 50$ and show the running time for different $n \in \{50, 100, 250, 500, 1000\}.$ We further test the RBCD and the RABCD algorithms in large scale cases. In Table \ref{tab:hctimenequald}, we set $n = d$ and show the running time for different $n, d \in \{20, 50, 100, 250, 500\}.$ We test the case when $n = 10 d$ with $d \in \{10, 20, 50, 100, 250\}$ in Table \ref{tab:hctimenequaltend}. We run each $n, d$ pair for 50 times and take the average. From Tables \ref{tab:hctimefixn} - \ref{tab:hctimenequaltend}, we see that our RBCD algorithm runs faster than the RGAS algorithm and our RABCD algorithm runs faster than the RAGAS algorithm in all cases. Moreover, we found that the advantage of RBCD (resp. RABCD) over RGAS (resp. RAGAS) is more significant when $n$ is relatively larger than $d$.
Moreover, the four algorithms for the PRW model are faster than the FW algorithm for computing the SRW distance.

\begin{table}[ht]\small
\centering\hspace*{-3.5em}
\begin{tabular}{|c|ccccc|}  \hline
Dimension $d$ & 20& 50& 100& 250& 500 \\ \hline
RBCD & \textbf{0.14} & \textbf{0.20} & \textbf{0.39} & \textbf{1.70} & \textbf{4.41}  \\ \hline
RGAS & 0.37 & 0.42 & 0.66 & 1.92 & 4.55  \\ \hline\hline
RABCD & \textbf{0.10} & \textbf{0.09} & \textbf{0.16} & \textbf{0.77} & \textbf{3.14}  \\ \hline
RAGAS & 0.27 & 0.23 & 0.23 & 0.85 & 3.20  \\ \hline\hline
FW & 1.42 & 1.82 & 2.71 & 8.88 & 24.25  \\ \hline
\end{tabular}
\vspace*{-.5em}
\caption{CPU time for calculating PRW of the fragmented hypercube problem with RBCD, RABCD, RGAS, RAGAS, and the FW algorithm for computing SRW. We set $n = 100$.}
\label{tab:hctimefixn}\vspace*{-.5em}
\end{table}

\begin{table}[ht]\small
\centering\hspace*{-3.5em}
\begin{tabular}{|c|ccccc|}  \hline
Number of points $n$ & 50& 100& 250& 500 &1000 \\ \hline
RBCD & \textbf{0.18} & \textbf{0.18} & \textbf{0.50} & \textbf{1.83} &\textbf{ 8.51}  \\ \hline
RGAS & 0.33 & 0.40 & 1.13 & 2.90 & 10.25  \\ \hline\hline
RABCD & \textbf{0.08} & \textbf{0.09} & \textbf{0.23} & \textbf{0.81} & \textbf{3.85}  \\ \hline
RAGAS & 0.17 & 0.21 & 0.61 & 1.48 & 5.39  \\ \hline\hline
SRW(FW)  & 1.24 & 1.81 & 4.58 & 15.42 & 64.65  \\ \hline
\end{tabular}
\vspace*{-.5em}
\caption{CPU time for calculating PRW of the fragmented hypercube problem with RBCD, RABCD, RGAS, RAGAS, and the FW algorithm for computing SRW. We set $d = 50$.}\label{tab:hctimefixd}\vspace*{-.5em}
\end{table}

\begin{table}[ht]\small
\centering\hspace*{-3.5em}
\begin{tabular}{|c|ccccc|} \hline
Dimension $d$ & 20& 50& 100& 250& 500 \\ \hline
RBCD & \textbf{0.06} & \textbf{0.16} & \textbf{0.35} & \textbf{2.62} & \textbf{12.75}  \\ \hline
RGAS & 0.18 & 0.30 & 0.61 & 3.20 & 13.12  \\ \hline\hline
RABCD & \textbf{0.06} & \textbf{0.08} & \textbf{0.12} & \textbf{1.14} & \textbf{7.97}  \\ \hline
RAGAS & 0.16 & 0.16 & 0.21 & 1.40 & 8.22  \\ \hline\hline
SRW(FW)  & 0.56 & 1.32 & 2.84 & 14.09 & 50.72  \\ \hline
\end{tabular}
\vspace*{-.5em}
\caption{CPU time for calculating PRW of the fragmented hypercube problem with RBCD, RABCD, RGAS, RAGAS, and the FW algorithm for computing SRW. We set $n=d$.}\label{tab:hctimenequald}\vspace*{-.5em}
\end{table}

\begin{table}[ht]\small
\centering\hspace*{-3.5em}
\begin{tabular}{|c|ccccc|} \hline
Dimension $d$ &10 & 20& 50& 100& 250\\ \hline
RBCD & \textbf{0.16} & \textbf{0.45} & \textbf{1.92} & \textbf{11.97} & \textbf{354.91}  \\ \hline
RGAS & 0.66 & 1.75 & 4.66 & 16.58 & 427.24  \\ \hline\hline
RABCD & \textbf{0.18} & \textbf{0.35} & \textbf{0.92} & \textbf{4.35} & \textbf{129.90}  \\ \hline
RAGAS & 0.79 & 1.26 & 2.35 & 7.22 & 157.07  \\ \hline\hline
SRW & 1.86 & 3.88 & 18.47 & 90.83 & 1355.86 \\ \hline
\end{tabular}
\vspace*{-.5em}
\caption{CPU time for calculating PRW of the fragmented hypercube problem with RBCD, RABCD, RGAS, RAGAS, and the FW algorithm for computing SRW. We set $n = 10 d$.}\label{tab:hctimenequaltend}\vspace*{-.5em}
\end{table}

We then repeat the above experiments on computing the PRW between two Gaussian distributions. We set $k^* = k = 5 $ for the Wishart distribution and the thresholds $\epsilon_{RGAS} =\epsilon_{RAGAS} = \epsilon_1 = \epsilon = 0.1$ and $\epsilon_2 = \epsilon_1^2$. We set the step size  $\tau_{RGAS} = \tau_{RBCD}/\eta,   \tau_{RAGAS} = \tau_{RABCD}/\eta $. Tables \ref{tab:gausstimefixn} - \ref{tab:gausstimenequaltend} give the computational time comparison for computing the PRW between two Gaussian distributions. We notice that tuning parameters $\eta$ and $ \tau_{RBCD}, \tau_{RABCD}$ to guarantee that the algorithms achieve their best performance for the Gaussian distributions is much more difficult that the Fragmented Hypercube example. We thus listed $(\eta, \tau_{RBCD}, \tau_{RABCD})$ for different $(n, d)$ pairs in the captions of Tables \ref{tab:gausstimefixn} - \ref{tab:gausstimenequaltend}. How to choose these parameters more systematically is an important topic for future study. We also use different step size for RBCD and RABCD algorithms so that both algorithms converge at their fastest speed.  Tables \ref{tab:gausstimefixn} - \ref{tab:gausstimenequaltend} show that our proposed RBCD algorithm runs faster than the RGAS algorithm and the proposed RABCD algorithm runs faster than the RAGAS algorithm in all tested cases for the Gaussian distributions.

\begin{table}[ht]\small
\centering\hspace*{-3.5em}
\begin{tabular}{|c|ccccc|} \hline
Dimension $d$ & 25& 50& 100& 250& 500 \\ \hline
RBCD & \textbf{0.73} & \textbf{0.76} & \textbf{0.62} & \textbf{2.48} & \textbf{3.46}  \\ \hline
RGAS & 1.05 & 0.93 & 0.81 & 2.65 & 3.63  \\ \hline\hline
RABCD &\textbf{ 0.16} & \textbf{0.67} & \textbf{0.63} & \textbf{2.20} & \textbf{4.67}  \\ \hline
RAGAS & 0.21 & 0.81 & 0.73 & 2.43 & 4.73  \\ \hline\hline
SRW(FW)   & 5.74 & 6.57 & 10.18 & 28.96 & 79.20  \\ \hline
\end{tabular}
\vspace*{-.5em}
\caption{CPU time for calculating PRW between two Gaussian distributions with RBCD, RABCD, RGAS, RAGAS, and the FW algorithm for computing SRW. We set $n = 100$ and $\tau_{RBCD} = 0.01$, $\tau_{RABCD} = 0.05$, $\eta = 10$ when $d < 50$; $\tau_{RABCD} = 0.1$, $\eta = 20$ when $d < 250$; $\tau_{RABCD} = 0.2$, $\eta = 100$ otherwise.}\label{tab:gausstimefixn}\vspace*{-.5em}
\end{table}

\begin{table}[ht]\small
\centering\hspace*{-3.5em}
\begin{tabular}{|c|ccccc|} \hline
Number of points $n$ & 25& 50& 100& 250& 500 \\ \hline
RBCD &\textbf{ 0.21} & \textbf{0.57} & \textbf{0.61} & \textbf{1.68} & \textbf{6.41}  \\ \hline
RGAS & 0.27 & 0.76 & 0.87 & 2.24 & 7.31  \\ \hline\hline
RABCD & \textbf{0.09} & \textbf{0.23} & \textbf{0.26} & \textbf{0.61} & \textbf{2.31}  \\ \hline
RAGAS & 0.14 & 0.34 & 0.36 & 0.80 & 2.69  \\ \hline\hline
SRW(FW)  & 1.63 & 3.78 & 5.09 & 12.57 & 41.93 \\ \hline
\end{tabular}
\vspace*{-.5em}
\caption{CPU time for calculating PRW between two Gaussian distributions with RBCD, RABCD, RGAS, RAGAS, and the FW algorithm for computing SRW. We set $d = 20$ and $\tau_{RBCD} = \tau_{RABCD} = 0.01$, $\eta = 10$.}\label{tab:gausstimefixd}\vspace*{-.5em}
\end{table}

\begin{table}[ht]\small
\centering\hspace*{-3.5em}
\begin{tabular}{|c|ccccc|} \hline
Dimension $d$ & 20& 50& 100& 250& 500 \\ \hline
RBCD & \textbf{ 0.15 }& \textbf{ 1.37} & \textbf{ 0.74} & \textbf{ 12.52 }& \textbf{ 56.46}  \\ \hline
RGAS & 0.36 & 2.10 & 1.31 & 13.58 & 58.22  \\ \hline\hline
RABCD & \textbf{ 0.05 }&\textbf{  0.97} & \textbf{ 0.62 }&\textbf{  9.22 }& \textbf{ 68.03}  \\ \hline
RAGAS & 0.12 & 1.45 & 1.02 & 10.04 & 72.76  \\ \hline\hline
SRW & 2.44 & 7.59 & 14.43 & 55.99 & 224.34  \\ \hline
\end{tabular}
\vspace*{-.5em}
\caption{CPU time for calculating PRW between two Gaussian distributions with RBCD, RABCD, RGAS, RAGAS, and the FW algorithm forcomputing SRW. We set $n=d$ and $\tau_{RBCD} = 0.01$, $\tau_{RABCD} = 0.05$, $\eta = 10$ when $d < 50$; $\tau_{RABCD} = 0.1$, $\eta = 20$ when $d < 250$; $\tau_{RABCD} = 0.2$, $\eta = 50$ otherwise. }\label{tab:gausstimenequald}\vspace*{-.5em}
\end{table}

\begin{table}[ht]\small
\centering\hspace*{-3.5em}
\begin{tabular}{|c|ccccc|} \hline
Dimension $d$ &15&  25& 50& 100& 250 \\ \hline
RBCD & \textbf{ 0.90} & \textbf{ 1.46 }& \textbf{ 5.74 }& \textbf{ 17.05 }  & \textbf{ 738.60} \\ \hline
RGAS & 3.16 & 3.21 & 6.99 & 19.20   &821.49 \\ \hline \hline
RABCD & \textbf{ 0.17 }& \textbf{ 0.57 }& \textbf{ 3.90} & \textbf{ 15.55 } & \textbf{ 518.75} \\ \hline
RAGAS & 0.38 & 0.75 & 4.26 & 16.24  &540.32 \\ \hline\hline
SRW & 4.15 & 11.67 & 42.92 & 180.62 & 2473.45  \\ \hline
\end{tabular}
\vspace*{-.5em}
\caption{CPU time for calculating PRW between two Gaussian distributions with RBCD, RABCD, RGAS, RAGAS, and the FW algorithm for computing SRW. We set $n = 10d$ and $\tau_{RBCD} = 0.01$, $\tau_{RABCD} = 0.05$, $\eta = 10$ when $d < 50$; $\tau_{RABCD} = 0.1$, $\eta = 20$ when $d < 250$; $\tau_{RABCD} = 0.2$, $\eta = 50$ otherwise. }\label{tab:gausstimenequaltend}
\end{table}

\subsection{Real Datasets}

In this section, we conduct experiments on two real datasets. The first one is a dataset with movie scripts that was used in \cite{paty2019subspace,lin2020projection}. More specifically, we first compute the PRW distances between each pair of movies in a corpus of seven movie scripts \cite{paty2019subspace, lin2020projection}, where each script is transformed into a list of words. We then use word2vec \cite{mikolov2018advances} to transform each script into a measure over $\br^{300}$ with the weights corresponding to the frequency of the words. We then compute the PRW distances between a preprocessed corpus of six Shakespeare operas. For both experiments, we set the parameters as $\eta = 0.1, \tau_{RBCD} = 0.1, \epsilon = 0.001, \tau_{RGAS} = \tau_{RBCD} /\eta$ and project each point onto a 2-dimensional subspace. We run each experiments for 10 times and take the average running time. In Tables \ref{tab:movie} and \ref{tab:opera}, the upper right half is the running time in seconds for RGAS/RBCD algorithms and the bottom left half is the $\PCal_k^2$ distance calculated by RGAS/RBCD algorithms. We highlight the smaller computational time in each upper right entry and the minimum PRW distance in each bottom left row. We see that the PRW distances are consistent and the RBCD algorithm runs faster than the RGAS algorithm in almost all cases.

\begin{table}[ht]\small
\centering\hspace*{-3.5em}
\begin{tabular}{|c|ccccccc|} \hline
& D & G & I & KB1 & KB2 & TM & T \\ \hline
D & -/-  & 7.13/\textbf{6.01} & \textbf{8.64}/9.03 & 6.15/\textbf{5.52} & 8.69/\textbf{7.99} & 7.62/\textbf{6.60} & 11.05/\textbf{10.24} \\ \hline
G & 0.129/0.129 & -/-  & 14.79/\textbf{12.68}  & 7.15/\textbf{5.95} & 8.48/\textbf{7.13}  & 13.42/\textbf{11.06}  & 18.36/\textbf{16.24}  \\ \hline
I & 0.135/0.135 & 0.102/0.102 &  -/- & 37.98/\textbf{32.06}  & 9.47/\textbf{7.99} & 17.46/\textbf{14.80} & 54.54/\textbf{49.46} \\ \hline
KB1 & 0.151/0.151 & 0.146/0.146 & 0.195/0.155 &-/- & 7.83/\textbf{6.87} & 10.47/\textbf{8.91}  & 30.83/\textbf{21.55} \\ \hline
KB2 & 0.161/0.161 & 0.157/0.157 & 0.166/0.166 & \textbf{0.088}/\textbf{0.088} & -/- &  9.69/\textbf{8.47} & 11.25/\textbf{9.23} \\ \hline
TM & 0.137/0.137 & \textbf{0.098}/\textbf{0.098} & \textbf{0.099}/\textbf{0.099} & 0.146/0.146 & 0.152/0.152 & -/- & 27.15/\textbf{25.13} \\ \hline
T & \textbf{0.103}/\textbf{0.103} & 0.128/0.128 & 0.135/0.135 & 0.136/0.136 & 0.138/0.138 & 0.134/0.134 &  -/- \\ \hline
\end{tabular}
\vspace*{-.5em}
\caption{\small{Each entry of the \textbf{Bottom Left half} is the $\PCal_k^2$ distance calculated by RGAS/RBCD algorithms between different movie scripts. Each entry of the \textbf{Upper Right half} is the running time in seconds for RGAS/RBCD algorithms between different movie scripts. D = Dunkirk, G = Gravity, I = Interstellar, KB1 = Kill Bill Vol.1, KB2 = Kill Bill Vol.2, TM = The Martian, T = Titanic.}}\label{tab:movie}\vspace*{-.5em}
\end{table}

\begin{table}[ht]\small
\centering\hspace*{-3.5em}
\begin{tabular}{|c|cccccc|} \hline
& H5 & H & JC & TMV & O & RJ \\ \hline
H5 & -/-  & 56.5/\textbf{44.48} & 6.63/\textbf{4.81} & 19.87/\textbf{15.69} & 25.91/\textbf{20.13} & 14.06/\textbf{4.96}\\ \hline
H & 0.123/0.123 & -/-  & 18.97/\textbf{15.19}  & 22.11/\textbf{20.54} & 14.65/\textbf{9.22}  & \textbf{17.53}/20.34    \\ \hline
JC & \textbf{0.117}/\textbf{0.117} & 0.127/0.126 &  -/- & 5.67/\textbf{4.72}  & 6.92/\textbf{5.35} & 4.35/\textbf{4.10}  \\ \hline
TMV & 0.134/0.134 & 0.112/0.112 & 0.094/0.093 &-/- & 8.43/\textbf{6.65} & 13.75/\textbf{10.67}   \\ \hline
O &  0.125/ 0.124 &  \textbf{0.091}/ \textbf{0.091} &  \textbf{0.086}/ \textbf{0.086} & \textbf{0.090}/\textbf{0.090} & -/- &  4.88/\textbf{4.17} \\ \hline
RJ & 0.239/0.239 & 0.249/0.249 & 0.172/0.172 & 0.226/0.226 & 0.185/0.185 & -/-  \\ \hline
\end{tabular}
\vspace*{-.5em}
\caption{\small{Each entry of the \textbf{Bottom Left half} is the $\PCal_k^2$ distance calculated by RGAS/RBCD algorithms between different Shakespeare plays. Each entry of the \textbf{Upper Right half} is the running time in seconds for RGAS/RBCD algorithms between different Shakespeare plays. H5 = Henry V, H = Hamlet, JC = Julius Caesar, TMV = The Merchant of Venice, O = Othello, RJ = Romeo and Juliet. (Note that the PRW distances are different from those reported in \cite{lin2020projection}. This is because we use a smaller $\eta$.)}}\label{tab:opera}\vspace*{-.5em}
\end{table}

We then conduct further experiments on the MNIST dataset. Specifically, we extract the 128-dimensional features of each digit from a pre-trained convolutional neural network, which achieves an accuracy of $98.6\%$ on the test set. Our task here is to compute the PRW distance by the RGAS and RBCD algorithms. We set parameters as $\eta = 8, \tau_{RBCD}= 0.004$ and $\tau_{RGAS} = \tau_{RBCD}/\eta$, $\epsilon = 0.1$ and compute the 2-dimensional projection distances between each pair of digits. All the distances are divided by 1000. We run the experiments for 10 times and take the average running time. In Table \ref{tab:mnist}, the upper right half is the running time in seconds for RGAS/RBCD algorithms and the bottom left half is the $\PCal_k^2$ distance calculated by RGAS/RBCD algorithms. We highlight the smaller computational time in each upper right entry and the minimum PRW distance in each bottom left row. We again observe that the PRW distances are consistent and the RBCD algorithm runs faster than the RGAS algorithm in almost all cases.

\begin{table}[ht]\footnotesize
\centering\hspace*{-6em}
\setlength{\tabcolsep}{0.1cm}{
\begin{tabular}{|c|cccccccccc|} \hline
&D0 &D1 &D2 &D3 &D4 &D5 &D6 &D7 &D8 &D9\\ \hline
D0 &-/-   & 15.50/\textbf{13.64} & 24.74/ \textbf{23.82} & 12.95/\textbf{8.91} & 21.91/\textbf{7.05} & 11.50/\textbf{6.99} & 15.66/\textbf{9.49}  & \textbf{12.93}/17.29 & 14.82/\textbf{12.36} & 12.30/\textbf{8.19} \\ \hline
D1 &0.98/0.98 & -/-  & \textbf{21.70}/30.00 & 30.09/\textbf{20.91} & 17.09/\textbf{13.72} & 31.06/\textbf{30.21} & \textbf{31.31}/37.00 & 45.75/\textbf{29.92} & 46.56/\textbf{44.88} & 20.12/\textbf{18.19} \\ \hline
D2 & \textbf{0.80}/ \textbf{0.80} &  \textbf{0.67}/ \textbf{0.66} & -/-   & \textbf{24.56}/35.84 & 26.15/\textbf{7.78} & 13.28/\textbf{8.58} & 20.43/\textbf{12.54} & 22.89/\textbf{9.40} & 23.78/\textbf{18.52} & 12.55/\textbf{8.19}\\ \hline
D3 &1.21/1.21 & 0.87/0.87& 0.73/0.72&-/- & 28.42/\textbf{18.37} & 15.81/\textbf{11.74} & 13.57/\textbf{ 9.77} & 14.08/\textbf{9.94} & 17.01/\textbf{15.09} & 32.50/\textbf{19.92} \\ \hline
D4 &1.24/1.24 & 0.67/0.67 & 1.09/1.09& 1.21/1.21& -/-  & 14.01/\textbf{11.15} & 28.69/\textbf{13.04} & 18.45/\textbf{12.14} & 13.07/\textbf{7.77} & 31.79/\textbf{22.33} \\ \hline
D5 &1.04/1.04 & 0.85/0.85& 1.09/1.09&  \textbf{0.59}/ \textbf{0.59}& 1.01/1.01&-/-  & 14.40/\textbf{13.54} & 19.82/\textbf{9.33} & 20.92/\textbf{13.51}& 18.58/\textbf{13.83} \\ \hline
D6 &0.81/0.81 & 0.80/0.80& 0.91/0.91& 1.24/1.24& 0.85/0.85 &  \textbf{0.72}/ \textbf{0.72}& -/-   & 13.89/\textbf{11.11} & 12.75/\textbf{8.46} & 14.14/\textbf{8.91}  \\ \hline
D7 &0.86/0.85 & \textbf{ 0.57}/ \textbf{0.58}& 0.70/0.71& 0.73/0.73& 0.80/0.80& 0.92/0.92& 1.11/1.11 & -/-  & 12.67/\textbf{7.43} & 28.14/\textbf{17.75} \\ \hline
D8 &1.06/1.06 & 0.88/0.88&  \textbf{0.68}/ \textbf{0.68}& 0.89/0.89& 1.10/1.10 & 0.72/0.72& 0.92/0.92 & 1.08/1.08 &-/-   & 30.87/\textbf{10.15} \\ \hline
D9 &1.09/1.09 & 0.86/0.86& 1.07/1.07& 0.84/0.84&  \textbf{0.50}/ \textbf{0.50}& 0.78/0.78 & 1.11/1.11 & 0.61/0.61& 0.87/0.87 &-/-   \\ \hline
\end{tabular}
}
\vspace*{-.1em}
\caption{\small{Each entry of the \textbf{Bottom Left half} is the $\PCal_k^2$ distance calculated by RGAS/RBCD algorithms for different pair of digits in the MNIST dataset. Each entry of the \textbf{Upper Right half} is the running time in seconds for RGAS/RBCD algorithms for different pair of digits in the MNIST dataset. (Note that the PRW distances are different from those reported in \cite{lin2020projection}. This is because we use different stopping criteria.)}}\label{tab:mnist}\vspace*{-.5em}
\end{table}

\begin{remark}
In our numerical experiments, we found that both RBCD and RGAS are sensitive to parameter $\eta$. This phenomenon was also observed when the Sinkhorn's algorithm was applied to solve the $\regot$ problem \cite{cuturi2013sinkhorn}. Roughly speaking, if $\eta$ is too small, then it may cause numerical instability, and if $\eta$ is too large, then the solution to $\regot$ is far away from the solution to the original OT problem. Moreover, the adaptive algorithms RABCD and RAGAS are also sensitive to the step size $\tau$, though they are usually faster than their non-adative versions RBCD and RGAS. We have tried our best to tune these parameters during our experiments so that the best performance is achieved for each algorithm. How to tune these parameters more systematically is left as a future work.
\end{remark}

\section{Conclusion}\label{sec:con}

In this paper, we {have} proposed RBCD and RABCD algorithms for computing the projection robust Wasserstein distance. Our algorithms are based on a novel reformulation to the regularized OT problem. We {have} analyzed the iteration complexity of both RBCD and RABCD algorithms, and this kind of complexity result seems to be new for BCD algorithm on Riemannian manifolds. Moreover, the complexity of arithmetic operations of our RBCD and RABCD algorithms is significantly better than that of the RGAS and RAGAS algorithms. We {have} conducted extensive numerical experiments and the results showed that our methods are more efficient than existing methods. Future work includes better tuning strategies of some parameters used in the algorithms.

\section*{Acknowledgements}
The authors thank Tianyi Lin for fruitful discussions on this topic and Meisam Razaviyayn for insightful suggestions on notions of the $\epsilon$-stationary point of min-max problem. This work was supported in part by NSF HDR TRIPODS grant CCF-1934568, NSF grants CCF-1717943, CNS-1824553, CCF-1908258, ECCS-2000415, DMS-1953210 and CCF-2007797, and UC Davis CeDAR (Center for Data Science and Artificial Intelligence Research) Innovative Data Science Seed Funding Program.

\bibliographystyle{plain}
\bibliography{RBCD}

\appendix
\section{On the Definition of $\epsilon$-stationary point}

In this section, we prove that for PRW \eqref{PRW}, our Definition \ref{def:primalsta} leads to the corresonding definition of $\epsilon$-stationary point in \cite{lin2020projection}. To this end, we only need to prove that \eqref{def:primalsta-eq-1} implies \eqref{def:eps-lin-1} when $\epsilon_1 =\epsilon$.

\begin{proof}
We assume that $(\hat{\pi},\hat{U})$ satisfies \eqref{def:primalsta-eq} with $\epsilon_1 =\epsilon_2=\epsilon$. Following the proof of Theorem 3.7 in \cite{lin2020projection-arxiv}\footnote{Here we refer to the version 5 of the arxiv paper \cite{lin2020projection-arxiv}.}, we denote $\pi^*$ as the projection of $\hat{\pi}$ onto the optimal solution set of the following OT problem:
\begin{equation}\label{appen-eps-OT}
\min_{\pi\in\Pi(\mu_n,\nu_n)} \langle \hat{U}\hat{U}^\top, V_{\pi} \rangle.
\end{equation}
Denote the optimal objective value of \eqref{appen-eps-OT} as $t^*$, the optimal solution set of \eqref{appen-eps-OT} is a polyhedron set:
\[\mathcal{S} = \{\pi \mid \pi\in\Pi(\mu_n,\nu_n), \langle \hat{U}\hat{U}^\top, V_{\pi} \rangle = t^*\}.\]
Note that the proof of Theorem 3.7 in \cite{lin2020projection-arxiv} also shows that
$$\text{subdiff} f(\hat{U}) \ni \proj_{T_{\hat{U}}\St} (2V_{\pi^*}\hat{U}). $$
Therefore, we have
\begin{align}\label{appen-eps-OT-2}
\text{dist}(0, \text{subdiff} f(\hat{U})) &\le \|\proj_{T_{\hat{U}}\St} (2V_{\pi^*}\hat{U}) \|_F\\
& \le \|\proj_{T_{\hat{U}}\St} (2V_{\pi^*}\hat{U} - 2V_{\hat{\pi}}\hat{U}) \|_F + \|\proj_{T_{\hat{U}}\St} (2V_{\hat{\pi}}\hat{U}) \|_F \nonumber\\
& \le 2\|(V_{\pi^*} - V_{\hat{\pi}})\hat{U}\|_F + \|\text{grad}_U f(  \hat{\pi}, \hat{U})\|_F \nonumber\\
& \le 2\|(V_{\pi^*} - V_{\hat{\pi}})\hat{U}\|_F + \epsilon \nonumber\\
& \le 2 \|C\|_\infty \|\pi^* - \hat{\pi}\|_1 + \epsilon,\nonumber
\end{align}
where the fourth inequality follows from \eqref{def:primalsta-eq-1} and the last inequality is due to the Cauchy-Schwarz inequality.
Now according to Lemma 3.6 in \cite{lin2020projection-arxiv}, there exists a constant $\theta>0$ such that
\be\label{appen-eps-OT-3}
\|\pi^* - \hat{\pi}\|_1 \le \theta \left\|\left\langle \hat{U}\hat{U}^\top, \frac{V_{\hat{\pi}} -V_{\pi^*} }{\|C\|_\infty} \right\rangle\right\|_1 \le \frac{\theta}{\|C\|_\infty} \epsilon.
\ee
where the second inequality is due to \eqref{def:primalsta-eq-2}. Substituting \eqref{appen-eps-OT-3} to \eqref{appen-eps-OT-2} yields
\[
\text{dist}(0, \text{subdiff} f(\hat{U})) \le (2\theta + 1)\epsilon,
\]
which completes the proof.
\end{proof}
\begin{remark}
We have proved that our Definition \ref{def:primalsta} leads to the corresponding definition of $\epsilon$-stationary point in \cite{lin2020projection} up to some constant that depends on $\theta$. Though $\theta$ may be large in practice, we point out that the convergence rate result in \cite{lin2020projection}[Theorem B.6] depends on the constant $\theta.$ As a contrast, by using our Definition \ref{def:primalsta}, our results are independent of $\theta.$
\end{remark}

\section{Additional Numerical Results}

\subsubsection{Computational Time Plot}
We further show how the proposed RBCD and RABCD algorithms evolve during the course of the algorithms. Specifically, we use
\[
f_\eta(U) = \sum_{ij}(\pi_\eta^*)_{ij}\|U^\top x_i - U^\top y_j\|^2,
\]
as a quality measure, where $\pi_\eta^*$ is the regularized optimal transport plan when fixing $U$. We plot $f_\eta(U)$ against the execution time for two synthetic datasets in Figure \ref{fig:f_time}. The results are averaged over 10 runs. In both figures, we see that our proposed two algorithms are always faster than their correspondents in \cite{lin2020projection} to achieve the same level of the quality measure.

\begin{figure}[ht]
\centering\hspace*{-1em}
\includegraphics[width=.45\textwidth]{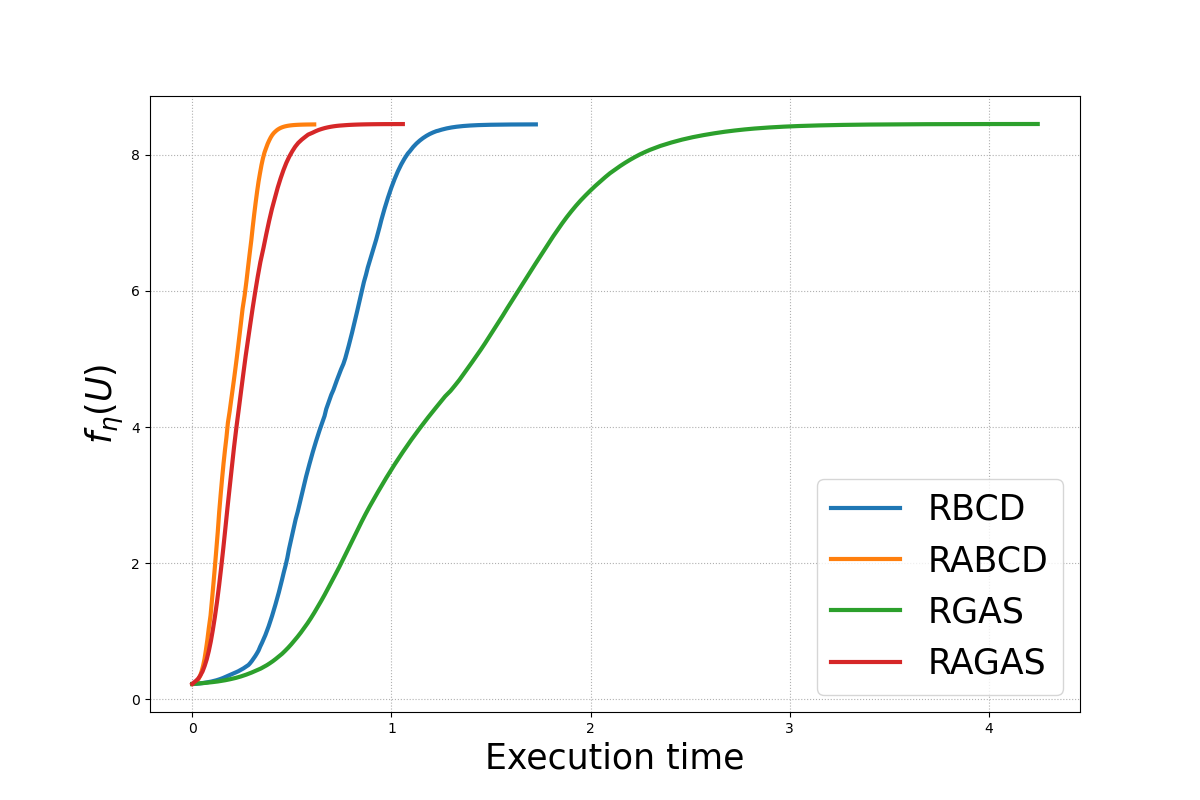}
\includegraphics[width=.45\textwidth]{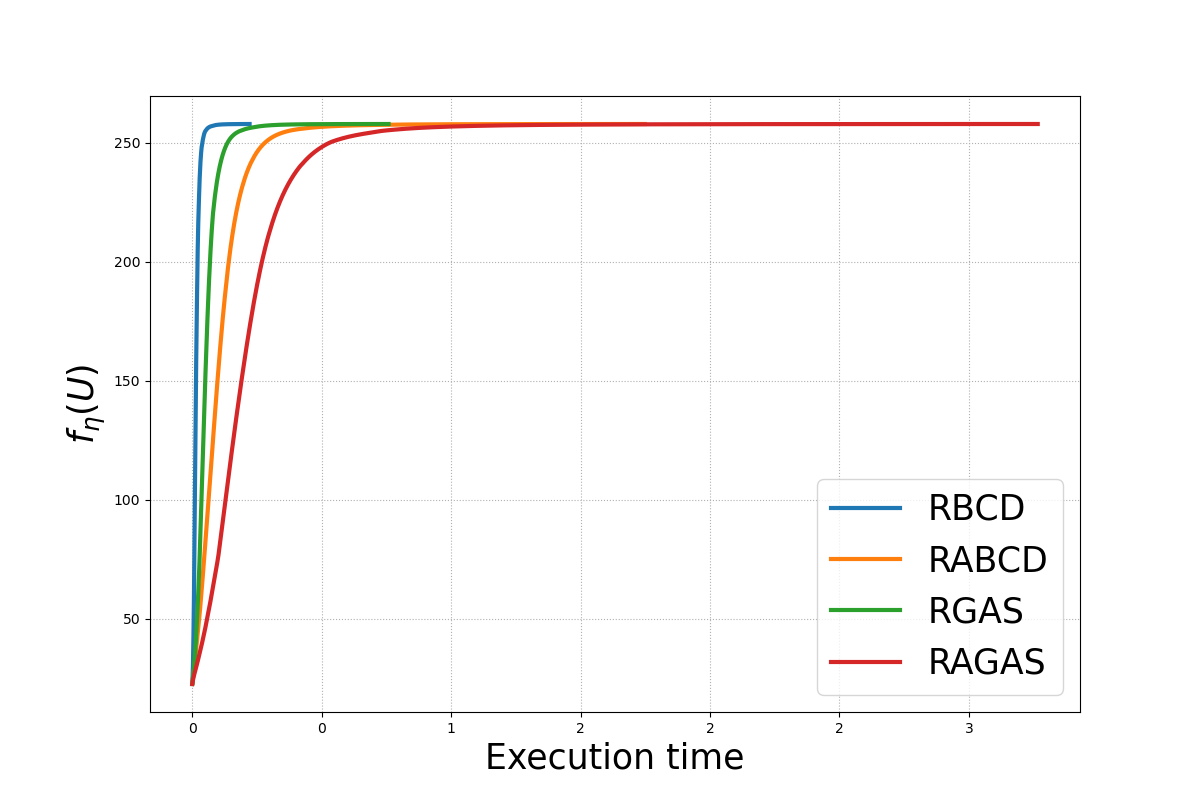}
\vspace*{-.5em}\caption{$f_\eta(U)$ against computational time. \textbf{Left:} Fragmented Hypercube, $d = 100, n = 500, k = k^*= 2, \eta = 0.2$. \textbf{Right:} Gaussian Distribution, $d = 50, n = 100, k = 5, k^*= 10, \eta = 10.$}
\label{fig:f_time}\vspace*{-1em}
\end{figure}

\subsection{Numerical Results for RABCD}\label{sec:numexpRABCD}

In this section, we provide more numerical results on the CPU time comparison for the RABCD algorithm and the RAGAS algorithm \cite{lin2020projection}.

\paragraph{Real Dataset:} We test the RABCD algorithm on the real datasets introduced in section \ref{sec:num}. We use the same process to transform the data into a measure over $\br^{300}$ with the weights corresponding to the frequency of the words. For the movie scripts dataset, we set the parameters as $\eta = 0.1, \tau_{RABCD} = 0.05, \epsilon = 0.001, \tau_{RAGAS} = \tau_{RABCD} /\eta$. For the Shakespeare's opera dataset, we set the parameters as $\eta = 0.1, \tau_{RABCD} = 0.0025, \epsilon = 0.001, \tau_{RAGAS} = \tau_{RABCD} /\eta$. We project each point onto a 2-dimensional subspace and run each experiments for 10 times and take the average running time. In Tables \ref{tab:movie_adap} and \ref{tab:opera_adap}, the upper right half is the running time in seconds for RAGAS/RABCD algorithms and the bottom left half is the $\PCal_k^2$ distance calculated by RAGAS/RABCD algorithms. We highlight the smaller computational time in each upper right entry and the minimum PRW distance in each bottom left row. We see that the PRW distances are consistent and the RABCD algorithm runs faster than the RAGAS algorithm in almost all cases.

\begin{table}[ht]\small
\centering
\begin{tabular}{|c|ccccccc|} \hline
& D & G & I & KB1 & KB2 & TM & T \\ \hline
D & 0/0 & 5.68/\textbf{4.52} & 6.96/\textbf{6.14} & \textbf{4.05}/5.50 & 6.16/\textbf{5.08} & 8.89/\textbf{7.74} & 21.11/\textbf{11.68}  \\ \hline
G & 0.129/0.129 & 0/0 & 33.01/\textbf{23.89} & 9.18/\textbf{7.82} & 5.55/\textbf{3.34}  & 18.58/\textbf{12.76}  & 28.51/\textbf{21.82}  \\ \hline
I & 0.137/0.137 & 0.102/0.102 & 0/0 & 8.17/\textbf{6.23} & 49.6/\textbf{7.11} & 19.41/\textbf{14.39} & 43.19/\textbf{33.85} \\ \hline
KB1 & 0.151/0.151 & 0.146/0.146 & 0.195/0.155 & 0/0 & 12.56/\textbf{8.99} & 7.12/\textbf{4.93} & 11.65/\textbf{9.59} \\ \hline
KB2 & 0.161/0.161 & 0.157/0.157 & 0.166/0.166 & \textbf{0.088}/\textbf{0.088} & 0/0 & 4.41/\textbf{3.54} & 15.75/\textbf{14.49} \\ \hline
TM & 0.137/0.137 & \textbf{0.098}/\textbf{0.098} & \textbf{0.099}/\textbf{0.099} & 0.146/0.146 & 0.152/0.152 & 0/0 & 41.05/\textbf{33.45} \\ \hline
T & \textbf{0.103}/\textbf{0.103} & 0.128/0.128 & 0.135/0.135 & 0.136/0.136 & 0.138/0.138 & 0.134/0.134 & 0/0 \\ \hline
\end{tabular}
\caption{\small{Each entry of the \textbf{Bottom Left half} is the $\PCal_k^2$ distance calculated by RAGAS/RABCD algorithms between different movie scripts. Each entry of the \textbf{Upper Right half} isthe running time in seconds for RAGAS/RABCD algorithms between different movie scripts. D = Dunkirk, G = Gravity, I = Interstellar, KB1 = Kill Bill Vol.1, KB2 = Kill Bill Vol.2, TM = The Martian, T = Titanic.}}\label{tab:movie_adap}
\end{table}

\begin{table}[h]\small
\centering\hspace*{-3.5em}
\begin{tabular}{|c|cccccc|} \hline
& H5 & H & JC & TMV & O & RJ \\ \hline
H5 & -/-  & 37.47/\textbf{28.67} & 29.48/\textbf{5.03} & 38.58/\textbf{23.39} & \textbf{32.98}/37.28 & 92.45/\textbf{66.24}\\ \hline
H & 0.123/0.123 & -/-  & 16.90/\textbf{13.15}  & 84.14/\textbf{46.70} & 34.14/\textbf{23.93}  &  103.45/\textbf{72.73}    \\ \hline
JC & \textbf{0.117}/\textbf{0.117} & 0.123/0.123 &  -/- & 16.42/\textbf{4.62}  & 19.18/\textbf{6.62} & 9.21/\textbf{6.12}  \\ \hline
TMV & 0.134/0.134 & 0.114/0.114 & 0.094/0.093 &-/- & 24.13/\textbf{14.28} & 71.42/\textbf{50.43}   \\ \hline
O &  0.125/ 0.124 &  \textbf{0.091}/ \textbf{0.091} &  \textbf{0.086}/ \textbf{0.086} & \textbf{0.090}/\textbf{0.090} & -/- &  13.41/\textbf{7.21} \\ \hline
RJ & 0.241/0.240 & 0.249/0.249 & 0.172/0.172 & 0.226/0.226 & 0.185/0.185 & -/-  \\ \hline
\end{tabular}
\vspace*{-.5em}
\caption{\small{Each entry of the \textbf{Bottom Left half} is the $\PCal_k^2$ distance calculated by RAGAS/RABCD algorithms between different Shakespeare plays. Each entry of the \textbf{Upper Right half} is the running time in seconds for RAGAS/RABCD algorithms between different Shakespeare plays. H5 = Henry V, H = Hamlet, JC = Julius Caesar, TMV = The Merchant of Venice, O = Othello, RJ = Romeo and Juliet. (Note that the PRW distances are different from those reported in \cite{lin2020projection}. This is because we use a smaller $\eta$.)}}\label{tab:opera_adap}\vspace*{-.5em}
\end{table}

\end{document}